\documentclass{article}

\usepackage[preprint,nonatbib]{neurips_2019}

\usepackage[numbers]{natbib}
\usepackage{enumitem}
\setlist{leftmargin=3mm}
\usepackage[utf8]{inputenc} 
\usepackage[T1]{fontenc}    
\usepackage{xr-hyper}
\usepackage[citecolor = blue,linkcolor = orange, filecolor = orange]{hyperref}       

\usepackage{url}            
\usepackage{booktabs}       
\usepackage{amsfonts}       
\usepackage{amsthm}
\let\oldproofname=\proofname
\renewcommand{\proofname}{\rm\bf{\oldproofname}}
\usepackage{nicefrac}       
\usepackage{microtype}      
\usepackage{mathtools}
\mathtoolsset{showonlyrefs}
\usepackage{floatrow}
\newfloatcommand{capbtabbox}{table}[][\FBwidth]

\newtheorem{theorem}{Theorem}
\newtheorem{definition}[theorem]{Definition}
\newtheorem{proposition}[theorem]{Proposition}
\newtheorem{lemma}[theorem]{Lemma}
\newtheorem{remark}[theorem]{Remark}
\newtheorem{corollary}[theorem]{Corollary}

\newtheorem{example}[theorem]{Example}

\newcommand{\commentout}[1]{}

\newcommand{\rmbf}[1]{\rm{\bf{#1}}}

\newcommand{\where}{\text{where }\ }

\newcommand*{\estimate}[1]{\hat{#1}}
\newcommand{\emperical}{L_n}

\newcommand{\comp}{\ensuremath{\text{\sc comp}}}

\newcommand{\risk}{\ensuremath{\bar{L}}}

\newcommand*{\op}[1]{\operatorname{#1}}

\newcommand*{\ceil}[1]{\lceil{#1}\rceil}
\newcommand*{\Ceil}[1]{\left\lceil{#1}\right\rceil}

\newcommand{\eP}{\ensuremath{\mathscr P}}
\newcommand{\eA}{\ensuremath{\mathscr A}}
\newcommand{\eC}{\ensuremath{\mathscr C}}

\newcommand{\cG}{\ensuremath{\mathcal G}}
\newcommand{\cH}{\ensuremath{\mathcal H}}

\newcommand{\cN}{\ensuremath{\mathcal N}}

\newcommand{\cX}{\ensuremath{\mathcal X}}
\newcommand{\cY}{\ensuremath{\mathcal Y}}
\newcommand{\cZ}{\ensuremath{\mathcal Z}}

\newcommand{\reals}{{\mathbb R}}

\newcommand{\ERM}{\ensuremath{\text{\sc ERM}}}

\newcommand{\Exp}{\ensuremath{\text{\rm E}}}
\newcommand{\expect}[2]{\mathbb{E}_{#1}\left[#2\right]}

\newcommand{\var}{\ensuremath{\text{\rm var}}}

\newcommand{\Bernoulli}{\ensuremath{\mathcal{B}}}

\newcommand{\nrenb}[1]{\ensuremath{D_{\eta \eta}}}

\newcommand{\E}{\mathbb{E}}

\newcommand{\rv}[1]{#1'}
\newcommand{\funky}[1]{\rv{\tilde{\genaction}}}
\newcommand{\funkyc}[1]{\rv{\tilde{\lambda}}}

\newcommand{\dol}{{{P}}} 
\newcommand{\KL}{\text{\sc KL}}
\newcommand{\kl}{ \operatorname{kl}}

\newcommand{\genaction}{\ensuremath{\theta}}

\renewcommand\Exp{\ensuremath{\mathbb E}}

\newcommand{\longp}[1]{}

\newcommand{\dist}{\ensuremath{\mathbf D}}

\usepackage[draft,multiuser]{fixme}
\fxusetheme{colorsig}
\FXRegisterAuthor{peter}{Peter}{\color{blue}PG} 
\FXRegisterAuthor{niche}{Niche}{\color{red}NM}
\newcommand{\notedone}[1]{}

\newcommand{\asversion}[1]{}

\newcommand*{\pipes}{\|}

\newcommand*{\argmin}{\arg\min}

\newcommand{\stochleq}{\leqclosed}

\renewcommand{\var}{\op{Var}}

\usepackage{mathtools}
\usepackage{graphics}
\usepackage{bm}
\usepackage{bbm}
\usepackage{epsf}
\usepackage{verbatim}
\usepackage{makeidx}
\usepackage{graphicx}
\usepackage{amsmath}

\usepackage{enumitem}
\usepackage{MnSymbol}
\usepackage{xcolor}
\usepackage{hyperref}
\hypersetup{colorlinks=true}
\usepackage{oubraces}
\usepackage{nicefrac} 
\usepackage{todonotes}
\usepackage{array}
\newcommand{\PreserveBackslash}[1]{\let\temp=\\#1\let\\=\temp}
\newcolumntype{C}[1]{>{\PreserveBackslash\centering}p{#1}}
\newcolumntype{R}[1]{>{\PreserveBackslash\raggedleft}p{#1}}
\newcolumntype{L}[1]{>{\PreserveBackslash\raggedright}p{#1}}

\usepackage[mathscr]{euscript}

\makeatletter
\newcommand*{\addFileDependency}[1]{
	\typeout{(#1)}
	\@addtofilelist{#1}
	\IfFileExists{#1}{}{\typeout{No file #1.}}
}
\makeatother

\title{PAC-Bayes Un-Expected Bernstein Inequality}
\author{%
	Zakaria Mhammedi \\
	The Australian National University and Data61\\
	\texttt{zak.mhammedi@anu.edu.au} \\
	\And
	Peter D. Gr\"unwald \\
	CWI and Leiden University \\
	\texttt{pdg@cwi.nl} \\
	\AND
	Benjamin Guedj \\
	Inria and University College London \\
	\texttt{benjamin.guedj@inria.fr} \\
}

\begin{document}
	
	\maketitle
	
	\begin{abstract}
		We present a new PAC-Bayesian generalization bound. Standard bounds contain a $\sqrt{L_n \cdot \KL/n}$ complexity term which dominates unless $L_n$, the empirical error of the learning algorithm's randomized predictions, vanishes. We manage to replace $L_n$ by a term which vanishes in many more situations, essentially whenever the employed learning algorithm is sufficiently stable on the dataset at hand. Our new bound consistently beats state-of-the-art bounds both on a toy example and on UCI datasets (with large enough $n$). Theoretically, unlike existing bounds, our new bound can be expected to converge to $0$ faster whenever a Bernstein/Tsybakov condition holds, thus
		connecting PAC-Bayesian generalization and {\em excess risk\/} bounds---for the latter it has long been known that faster convergence can be obtained under Bernstein conditions. Our main technical tool is a new concentration inequality which is like Bernstein's but with $X^2$ taken outside its expectation. 
	\end{abstract}
	\section{Introduction}
	PAC-Bayesian generalization bounds \citep{alquier2018simpler,catoni2003pac,Catoni07,germain2009pac,germain2015risk, guedj2019primer,McAllester98,McAllester99,McAllester02} have recently obtained renewed interest within the context of deep neural networks \citep{dziugaite2017computing,neyshabur2017pac,zhou2018}. In particular, Zhou et al. \citep{zhou2018} and Dziugaite and Roy \citep{dziugaite2017computing} showed that, by extending an idea due to Langford and Caruana \cite{LangfordC02}, one can obtain nontrivial (but still not very strong) generalization bounds on real-world datasets such as MNIST and ImageNet. Since using alternative methods, nontrivial generalization bounds are even harder to get, there remains a strong interest in improved PAC-Bayesian bounds. In this paper, we provide a considerably improved bound whenever the employed learning algorithm is sufficiently \emph{stable} on the given data.
	
	Most standard bounds have an order $\sqrt{L_n \cdot \comp_n/n}$ term on the right, where $\comp_n$ represents model complexity in the form of a Kullback-Leibler divergence between a prior and a posterior, and $L_n$ is the {\em posterior expected loss\/} on the training sample. The latter only vanishes if there is a sufficiently large neighborhood around the ``center'' of the posterior at which the training error is 0. In the two papers \cite{dziugaite2017computing,zhou2018} mentioned above, this is not the case. For example, the various deep net experiments reported by Dziugaite et al. \cite[Table 1]{dziugaite2017computing} with $n= 150000$ all have $L_n$ around $0.03$, so that $\sqrt{\comp_n/n}$ is multiplied by a non-negligible $\sqrt{0.03} \approx 0.17$. Furthermore, they have $\comp_n$ increasing substantially with $n$, making $\sqrt{L_n \cdot \comp_n/n}$ converge to $0$ at rate slower than $1/\sqrt{n}$.
	
	In this paper, we provide a bound (Theorem~\ref{thm:main}) with  $L_n$ replaced by a second-order term $V_n$---a term which will go to $0$ in many cases in which $L_n$ does not. This can be viewed as an extension of an earlier second-order approach by Tolstikhin and Seldin \cite{tolstikhin2013pac} (TS from now on); they also replace $L_n$, but by a term that, while usually smaller than $L_n$, will tend to be larger than our $V_n$. Specifically, as they write, in classification settings (our primary interest), their replacement is not much smaller than $L_n$ itself. Instead our $V_n$ can be very close to $0$ in classification even when $L_n$ is large. While the TS  bound is based on an ``empirical'' Bernstein inequality due to \cite{maurer2009empirical}\footnote{An alternative form of empirical Bernstein inequality appears in \cite{wintenberger2017}, based on an inequality due to \cite{cesa2007}.}, our bound is based on a different modification of  Bernstein's moment inequality in which the occurrence of $X^2$ is taken outside of its expectation (see Lemma~\ref{lem:bernsand}). We note that an empirical Bernstein inequality was introduced in \citep[Theorem 1]{audibert2007}, and the name ``Empirical Bernstein'' was coined in \cite{mnih2008}. 
	
	The term $V_n$ in our bound goes to $0$---and our bound improves on existing bounds---whenever the employed learning algorithm is relatively stable on the given data; for example, if the predictor learned on an initial segment (say, $50\%$) of the dataset performs similarly (\emph{i.e.} assigns similar losses to the same samples) to the predictor based on the full data. This improvement is reflected in our experiments where, except for very small sample sizes, we consistently outperform existing bounds both on a toy classification problem with label noise and on standard UCI datasets \cite{Dua:2019}. 
	Of course, the importance of stability for generalization has been recognized before in landmark papers such as \cite{bousquet2002stability,mukherjee2006learning,shalev2010learnability}, and recently also in the context of PAC-Bayes bounds \cite{rivasplata2018pac}. However, the data-dependent stability notion ``$V_n$'' occurring in our bound seems very different from any of the notions discussed in those papers.
	
	Theoretically, a further contribution is that we connect our PAC-Bayesian generalization bound to {\em excess risk bounds\/}; we show that (Theorem~\ref{thm:vnbound}) our generalization bound can be of comparable size to excess risk bounds up to an irreducible {\em complexity-free\/} term that is independent of model complexity. The excess risk bound that can be attained for any given problem depends both on the complexity of the set of predictors ${\cal H}$ and on the inherent ``easiness'' of the problem. The latter is often measured in terms of the exponent $\beta\in [0,1]$ of the {\em Bernstein condition\/} that holds for the given problem
	\citep{bartlett2006empirical,erven2015fast,grunwald2016fast}, 
	which generalizes the exponent in the celebrated {\em Tsybakov margin condition} \citep{bartlett2006convexity,tsybakov2004optimal}. The larger $\beta$, the faster the excess risk converges. In Section \ref{sec:bernstein}, we essentially show that the rate at which the
	$\sqrt{V_n\cdot \comp_n/n}$ term goes to $0$ can also be bounded by a quantity that gets smaller as $\beta$ gets larger. In contrast, previous PAC-Bayesian bounds do not have such a property. 
	
	\textbf{Contents.}
	In Section~\ref{sec:notation}, we introduce the problem setting and provide a first, simplified version of our main theorem. Section~\ref{sec:mainresults} gives our main bound. Experiments are presented in Section~\ref{sec:experiments}, followed by theoretical motivation in Section~\ref{sec:bernstein}. 
	The proof of our main bound is provided in Section~\ref{sec:esi}, where we first present the convenient ESI language for expressing stochastic inequalities, and (our main tool) the unexpected Bernstein lemma (Lemma~\ref{lem:bernsand}). The paper ends with an outlook for future work.
	
	\section{Problem Setting, Background, and Simplified Version of Our Bound}\label{sec:notation}
	\paragraph{Setting and Notation.} 
	Let $Z_1, \ldots, Z_n$ be i.i.d. random variables in some set $\cZ$, with $Z_1 \sim \dist$. Let $\cH$ be a hypothesis set and $\ell: \cH \times {\cal Z} \rightarrow [0,b]$, $b>0$, be a bounded loss function such that $\ell_h(Z) := \ell(h,Z)$
	denotes the loss that hypothesis $h$ makes on $Z$. We call any such tuple $(\dist,\ell,\cH)$ a {\em learning problem}. For a given hypothesis $h\in \cH$, we denote its \emph{risk} (expected loss on a test sample of size 1) by $L(h) \coloneqq \expect{Z\sim \dist}{\ell_h(Z)}$ and its empirical error by $\emperical(h) \coloneqq \frac{1}{n}\sum_{i=1}^n \ell_h(Z_i)$. For any distribution $\dol$ on $\cH$, we write $L(\dol) \coloneqq \expect{h\sim \dol}{L(h)}$ and $L_n(\dol) \coloneqq \expect{h\sim \dol}{L_n(h)}$. 
	
	For any $m\in[n]$ and any variables $Z_1,\dots,Z_n$ in $\mathcal{Z}$, we denote $Z_{\leq m}\coloneqq (Z_1,\dots,Z_m)$ and $Z_{<m}\coloneqq Z_{\leq m-1}$, with the convention that $Z_{\leq 0}=\emptyset$. Similarly, we denote $Z_{\geq m}\coloneqq (Z_m,\dots,Z_n)$ and $Z_{>m}\coloneqq Z_{\geq m+1}$, with the convention that $Z_{\geq {n+1}}=\emptyset$. As is customary in PAC-Bayesian works, a {\em learning algorithm\/} is a (computable) function $P: \bigcup_{i=1}^n \cZ^i\rightarrow  \mathcal{P}(\cH)$ that, upon observing input $Z_{\leq n}\in \mathcal{Z}^n$, outputs a ``posterior'' distribution $\dol(Z_{\leq n})(\cdot)$ on $\cH$\commentout{ (and not just a single element of $\cH$)}. The posterior could be a Gibbs or a generalized-Bayesian posterior but also other algorithms.  When no confusion can arise, we will abbreviate $\dol(Z_{\leq n})$ to $\dol_n$, and denote $\dol_0$ any ``prior'' distribution, \emph{i.e.} a distribution on $\cH$ which has to be specified in advance, before seeing the data; we will use the convention $P(\emptyset)=P_0$. Finally, we denote the Kullback-Leibler divergence between $\dol_n$ and $\dol_0$ by $\KL(\dol_n\pipes\dol_0)$.
	
	\textbf{Comparing Bounds.} Both existing state-of-the-art PAC-Bayes bounds and ours essentially take the following form; there exists constants $\eP, \eA,\eC \geq 0$, and a function $\varepsilon_{\delta,n}$, logarithmic in $1/\delta$ and $n$, such that for all $\delta \in ]0,1[$, with probability at least $1-\delta$ over the sample $Z_1,\ldots,Z_n$, it holds that,
	\begin{align}
	\label{eq:starter}
	L(P_n) - L_n(P_n)  \leq  \eP \cdot \sqrt{\frac{ R_n \cdot (\comp_n + \varepsilon_{\delta,n})}{n}} + \eA \cdot \frac{ \comp_n + \varepsilon_{\delta,n} }{n} + \eC \cdot \sqrt{\frac{ R'_n \cdot \varepsilon_{\delta,n}}{n}},
	\end{align}
	where $R_n,R_n'\geq 0$ are sample-dependent quantities which may differ from one bound to another. Existing classical bounds that after slight relaxations take on this form are due to Langford and Seeger \citep{langford2003pac,Seeger02}, Catoni \cite{catoni2007pac}, Maurer \cite{maurer2004note}, and Tolstikhin and Seldin (TS) \cite{tolstikhin2013pac} (see the latter for a nice overview). In all these cases, $\comp_n =  \KL(\dol_n \| \dol_0 )$, $R'_n= 0$, and---except for the TS bound---$R_n = L_n(P_n)$. For the TS bound, $R_n$ is equal to the empirical loss variance. Our bound in Theorem~\ref{thm:main} also fits (\ref{eq:starter}) (after a relaxation), but with considerably different choices for $\comp_n$, $R_n'$, and $R_n$.
	
	Of special relevance in our experiments is the bound due to Maurer \cite{maurer2004note}, which as noted by TS \cite{tolstikhin2013pac} tightens the PAC-Bayes-kl inequality due to Seeger \cite{seeger2002pac}, and is one of the tightest known generalization bounds in the literature. It can be stated as follows: for $\delta\in ]0,1[ $, $n\geq 8$, and any learning algorithm $P$, with probability at least $1-\delta$,
	\begin{align}
	\kl(L(P_n),L_n(P_n)) \leq  \frac{\KL(P_n\pipes P_0 )+ \ln \frac{2 \sqrt{n}}{\delta}}{n},    \label{eq:maurer}
	\end{align}
	where $\kl$ is the binary Kullback-Leibler divergence. Applying the inequality $p\leq q + \sqrt{2q \kl(p\pipes q)} + 2 \kl(p\pipes q)$ to \eqref{eq:maurer} yields a bound of the form \eqref{eq:starter} (see \cite{tolstikhin2013pac} for more details). 
	Note also that using Pinsker's inequality together with \eqref{eq:maurer} implies McAllester's classical PAC-Bayesian bound \cite{McAllester98}. 
	
	We now present a simplified version of our bound in Theorem~\ref{thm:main} below as a corollary.
	
	\begin{corollary}
		\label{cor:main}		 For any $1 \leq m<n$ and any deterministic estimator $\hat{h}: \bigcup_{i=1}^n \cZ^i\rightarrow \cH$ (such as \ERM{}), there exists $\eP, \eA, \eC > 0$, such that \eqref{eq:starter} holds with probability at least $1-\delta$, with  
		\begin{align}
		&  \comp_n = 
		\KL(\dol_n \| \dol(Z_{\leq m})) + \KL(\dol_n \| \dol(Z_{> m})), \label{eq:compterm}  \\
		&  R'_n \coloneqq  V_n' \coloneqq \frac{1}{n} 
		\sum_{i=1}^{m}\ell_{\hat{h}(Z_{>m})} (Z_i)^2
		+    \frac{1}{n} \sum_{j=m+1}^n\ell_{ \hat{h}(Z_{\leq m})}(Z_j)^2,  \label{eq:irred} \\ 
		&   R_n \coloneqq V_n  \coloneqq \frac{1}{n}  \expect{h \sim \dol_n}{
			\sum_{i=1}^{m} \left(\ell_{h}(Z_{i}) - \ell_{\hat{h}(Z_{>m})} (Z_i) \right)^2
			+       \sum_{j=m+1}^n \left(\ell_{h}(Z_j) - \ell_{ \hat{h}(Z_{\leq m})}(Z_j)\right)^2}.  \label{eq:theVn}
		\end{align}
	\end{corollary}
	Like in TS's and Catoni's bound, but unlike McAllester's and Maurer's, our $\varepsilon_{\delta,n}$ grows as $(\ln \ln n)/\delta$. Another difference is that our complexity term is a sum of two KL divergences, in which the prior (in this case $P(Z_{\leq m})$ or $P(Z_{>m})$) is ``informed''---when $m=n/2$, it is really the posterior based on half the sample. Our experiments confirm that this tends to be  much smaller than $\KL(P_n \| P_0)$. Other bounds can also be modified to make use of informed priors and replace the $\KL(P_n \| P_0)$ term by $\comp_n$ in \eqref{eq:compterm}. This is formalized in the next section. 
	
	A larger difference between our bound and others is in the fact that we have $R_n = V_n$ instead of the typical empirical error $R_n = L_n(P_n)$.  Only TS \cite{tolstikhin2013pac} have a $R_n$ that is somewhat reminiscent of ours; in their case $R_n = \Exp_{h \sim \dol_n}[\sum_{i=1}^n 
	\left(\ell_{h}(Z_{i}) - L_n(h) 
	\right)^2]/(n-1)$ is the empirical loss variance.
	The crucial difference to our $V_n$ is that the empirical loss variance cannot be close to $0$ unless a sizeable $P_n$-posterior region of $h$ has empirical error almost constant on most data instances. For classification with 0-1 loss, this is a strong condition since the empirical loss variance is equal to $n L_n(P_n)(1- L_n(P_n))/(n-1)$, which is only close to $0$ if $L_n(P_n)$ is itself close to $0$ or $1$. In contrast, our 
	$V_n$ can go to zero $0$ even if the empirical error and variance do not, as long as the learning algorithm is sufficiently stable. 
	This can be witnessed in our experiments in Section \ref{sec:experiments}. In Section \ref{sec:bernstein}, we argue more formally that under a Bernstein condition, the $\sqrt{V_n \cdot \comp_n/n}$ term in our bound can be much smaller than $\sqrt{\comp_n/n}$. Note, finally, that the term $V_n$ has a two-fold cross-validation flavor, but in contrast to a cross-validation error, for $V_n$ to be small, it is sufficient that the losses are {\em similar\/}, not that they are small.
	
	The price we pay for having $R_n = V_n$ in our bound is the right-most, irreducible remainder term in (\ref{eq:starter}) of order at most $b/\sqrt{n}$. Note, however, that this term is decoupled from the complexity $\comp_n$, and thus it is not affected by $\comp_n$ growing with the ``size'' of $\cH$. The following lemma gives a tighter bound (tighter than the $b/\sqrt{n}$ just mentioned) on the irreducible term:
	\begin{lemma}
		\label{lem:irreduce}
		Suppose that the loss is bounded by 1 (\emph{i.e. } $b=1$) and that $n$ is even, and let $m=n/2$. For $\delta \in]0,1[$, $R_n'$ as in \eqref{eq:irred}, and any estimator $\hat{h}: \bigcup_{i=1}^n \cZ^i\rightarrow \cH$, we have, with probability at least $1-\delta$,
		\begin{align}
		\sqrt{\frac{R'_n}{n}} \leq  \sqrt{\frac{2(L(\hat h(Z_{> m}))+ L(\hat h(Z_{\leq m})))}{ n}} + \frac{4\sqrt{\ln \frac{4}{\delta}}}{n} \label{eq:irrebound}.
		\end{align}
	\end{lemma}
	Behind the proof of the lemma is an application of Hoeffding's and the empirical Bernstein inequality \cite{maurer2009empirical} (see Section \ref{app:otherproofs}). Note that in the realizable setting, the first term on the RHS of \eqref{eq:irrebound} can be of order $O(1/n)$ with the right choice of estimator $\hat h$ (\emph{e.g.} ERM). In this case (still in the realizable setting), our irreducible term would go to zero at the same rate as other bounds which have $R_n= L_n(P_n)$.
	
	\section{Main Bound}
	\label{sec:mainresults}
	We now present our main result in its most general form. Let $\vartheta(\eta) \coloneqq  (-\ln (1-\eta) -\eta)/\eta^2$ and $c_{\eta}\coloneqq \eta \cdot \vartheta(\eta b)$, for $\eta \in]0,1/b[$, where $b>0$ is an upper-bound on the loss $\ell$.
	\begin{theorem}{\bf [Main Theorem]}
		\label{thm:main}
		Let $Z_1,\dots,Z_n$ be i.i.d. with $Z_1 \sim \dist$. Let $m\in [0..n]$ and $\pi$ be any distribution with support on a finite or countable grid $\cG \subset  ]0,1/b[$. For any $\delta\in ]0,1[$, and any learning algorithms $\dol, Q: \bigcup_{i=1}^n \mathcal{Z}^i\rightarrow \mathcal{P}(\cH)$, we have, 
		\begin{align} L(\dol_n) \leq  \emperical(\dol_n) + \inf_{\eta \in \cG } \left\{
		c_{\eta} \cdot  V_n
		+  \frac{\comp_n  +2 \ln \frac{1}{\delta \cdot \pi(\eta)}}{\eta \cdot n} \right\} + \inf_{\nu \in \cG } \left\{ c_{\nu} \cdot V'_n + \frac{\ln \frac{1}{\delta \cdot \pi(\nu)}}{\nu \cdot n }  \right\} , \label{eq:highprob}
		\end{align}
		with probability at least $1-\delta$,   where $\comp_n$, $V_n'$, and $V_n$ are the random variables defined by:
		\begin{align} 
		&  \comp_n \coloneqq 
		\KL(\dol_n \| \dol(Z_{\leq m})) + \KL(\dol_n \| \dol(Z_{> m})),  \label{eq:thecomp} \\
		&V'_n \coloneqq \frac{1}{n} \sum_{i=1}^m \expect{h\sim Q(Z_{>i}) }{ \ell_{h}(Z_i)^2} + \frac{1}{n}  \sum_{j=m+1}^n \expect{h\sim Q(Z_{<j}) }{\ell_{h}(Z_j)^2}, \nonumber
		\\ 
		&  V_n \coloneqq \frac{1}{n}  \expect{h \sim \dol_n}{
			\sum_{i=1}^{m} \left(\ell_{h}(Z_{i}) - \expect{\rv{h} \sim Q(Z_{>i})}{\ell_{\rv{h}} (Z_i)}\right)^2
			+       \sum_{j=m+1}^n \left(\ell_{h}(Z_j) - \expect{\rv{h} \sim Q(Z_{<j})}{\ell_{ \rv{h}} (Z_j)}\right)^2}. \nonumber
		\end{align}
	\end{theorem}
	While the result holds for all $0 \leq m \leq n$, in the remainder of this paper, we assume for simplicity that $n$ is even and that $m = n/2$. 
	We will also be using the grid $\cG$ and distribution $\pi$ defined by \begin{align} \label{eq:grid} \cG \coloneqq \left\{\tfrac{1}{2 b} , \dots, \tfrac{1}{2^K b}: K \coloneqq \Ceil{\log_2 \left(\tfrac{1}{2} \sqrt{\tfrac{n}{\ln \frac{1}{\delta} }}\right) }\right\},  \ \ \text{and} \ \ \text{ $\pi\equiv$ uniform distribution over $\cG$}. \end{align}
	Roughly speaking, this choice of $\cG$ ensures that the infima in $\eta$ and $\nu$ in \eqref{eq:highprob} are attained within $[\min \cG, \max \cG]$. Using the relaxation $c_{\eta}\leq \eta/2 +\eta^2 11b/20$, for $\eta \leq 1/(2b)$, in \eqref{eq:highprob} and tuning $\eta$ and $\nu$ within the grid $\cG$ defined in \eqref{eq:grid} leads to a bound of the form \eqref{eq:starter}.
	Furthermore, we see that the expression of $V_n$ in Corollary \ref{cor:main} now follows when $Q$ is chosen such that, for $1\leq i \leq m < j \leq n$, $Q(Z_{>i})\equiv \delta(\hat{h}(Z_{>m}))$ and $Q(Z_{<j })\equiv\delta(\hat{h}(Z_{\leq m}))$, for some deterministic estimator $\hat{h}$, where $\delta(h)(\cdot)$ denotes the Dirac distribution at $h\in \cH$. 
	
	\paragraph{Online Estimators.} It is clear that Theorem~\ref{thm:main} is considerably more general than its Corollary \ref{cor:main}; when predicting the $j$-th point $Z_j$, $j> m$, in the RHS sum of $V_n$, we could use a posterior $Q(Z_{<j })\equiv\delta(\hat{h}(Z_{< j}))$ which does not only depend on $Z_1,\ldots,Z_m$, but also on part of the second sample, namely $Z_{m+1}, \ldots, Z_{j-1}$, and analogously when predicting $Z_i$, $i\leq m$, in the LHS sum of $V_n$. We can thus base our bound on a sum of errors achieved by {\em online estimators} $(\hat{h}(Z_{< j}))$ and $(\hat h(Z_{> i}))$ which converge to the final $\hat{h}(Z_{\leq n})$ based on the full data. Doing this would likely improve our bounds, but we did not try it in our experiments since it is computationally demanding. 
	
	\paragraph{Informed Priors.} 
	Other bounds can also be modified to make use of ``informed
	priors'' from each half of the data; in this case, the
	$\KL(P_n\pipes P_0)$ term in these bounds can be replaced by
	$\comp_n$ defined in \eqref{eq:thecomp}. As revealed by
	additional experiments in the Appendix \ref{supp:moreexp},
	doing this substantially improves the corresponding bounds
	when the learning algorithm is sufficiently stable. Here we
	show how this can be done for Maurer's bound in
	\eqref{eq:maurer} (the details for other bounds are postponed
	to Appendix \ref{app:informedpriors}).
	\begin{lemma}
		\label{lem:fwdbwdmaurer}
		Let $\delta\in]0,1[$ and $m\in [0..n]$. In the setting of
		Theorem \ref{thm:main}, we have, with probability at least
		$1-\delta$,
		\begin{align}
		\kl(L(P_n),L_n(P_n)) \leq  \frac{	\KL(\dol_n \| \dol(Z_{\leq m})) + \KL(\dol_n \| \dol(Z_{> m}))+ \ln \frac{4 \sqrt{m(n-m)}}{\delta}}{n}.    \label{eq:maurerinformed}
		\end{align}
	\end{lemma}
	\begin{remark}
		\label{rmk:prior}
		(Useful for Section~\ref{sec:bernstein} below) Though this may deteriorate the bound in practice, Theorem~\ref{thm:main} allows choosing a learning algorithm $P$ such that for $1\leq m<n$, $P(Z_{\leq m})\equiv P(Z_{>m})\equiv P_0$ (\emph{i.e.} no informed priors); this results in $\comp_n= 2 \KL(P_n\pipes P_0)$---the bound is otherwise unchanged.
	\end{remark}	
	\paragraph{Biasing.} The term $V_n$ in our bound can be seen as the result of ``biasing'' the loss when evaluating the generalization error on each half of the sample. The TS bound, having a second order variance term, can be used in a way as to arrive at a bound like ours with the same $V_n$ as in Corollary \ref{cor:main}. The idea here is to apply the TS bound twice (once on each half of the sample) to the biased losses $\ell(h,\cdot)- \ell(\hat h(Z_{\leq m}), \cdot)$ and $\ell(h,\cdot)- \ell(\hat h(Z_{> m}), \cdot)$, then combine the results with a union bound. The details of this are postponed to Appendix \ref{app:biasing}. Note however, that this trick will not lead to a bound with a $V_n$ term as in Theorem \ref{thm:main}, \emph{i.e.} with the online posteriors $(Q(Z_{>i}))$ and $(Q(Z_{<j}))$ which get closer and closer to the final $Q(Z_{\leq m})$ based on the full sample. 
	
	\section{Experiments}
	\label{sec:experiments}
	In this section, we experimentally compare our bound in Theorem \ref{thm:main} to that of TS \cite{tolstikhin2013pac}, Catoni \cite[Theorem 1.2.8]{Catoni07} (with $\alpha=2$), and Maurer in \eqref{eq:maurer}. For the latter, given $L_n(P_n)\in [0,1[$ and the RHS of \eqref{eq:maurer}, we solve for an upper bound of $L(P_n)$ by ``inverting'' the $\kl$. We note that TS \cite{tolstikhin2013pac} do not claim that their bound is better than Maurer's in classification (in fact, they do better in other settings). 
	
	\begin{figure}[h]
		\begin{floatrow}
			\ffigbox{%
				\includegraphics[trim=0cm 0cm 0cm 0cm, clip, width=1\linewidth]{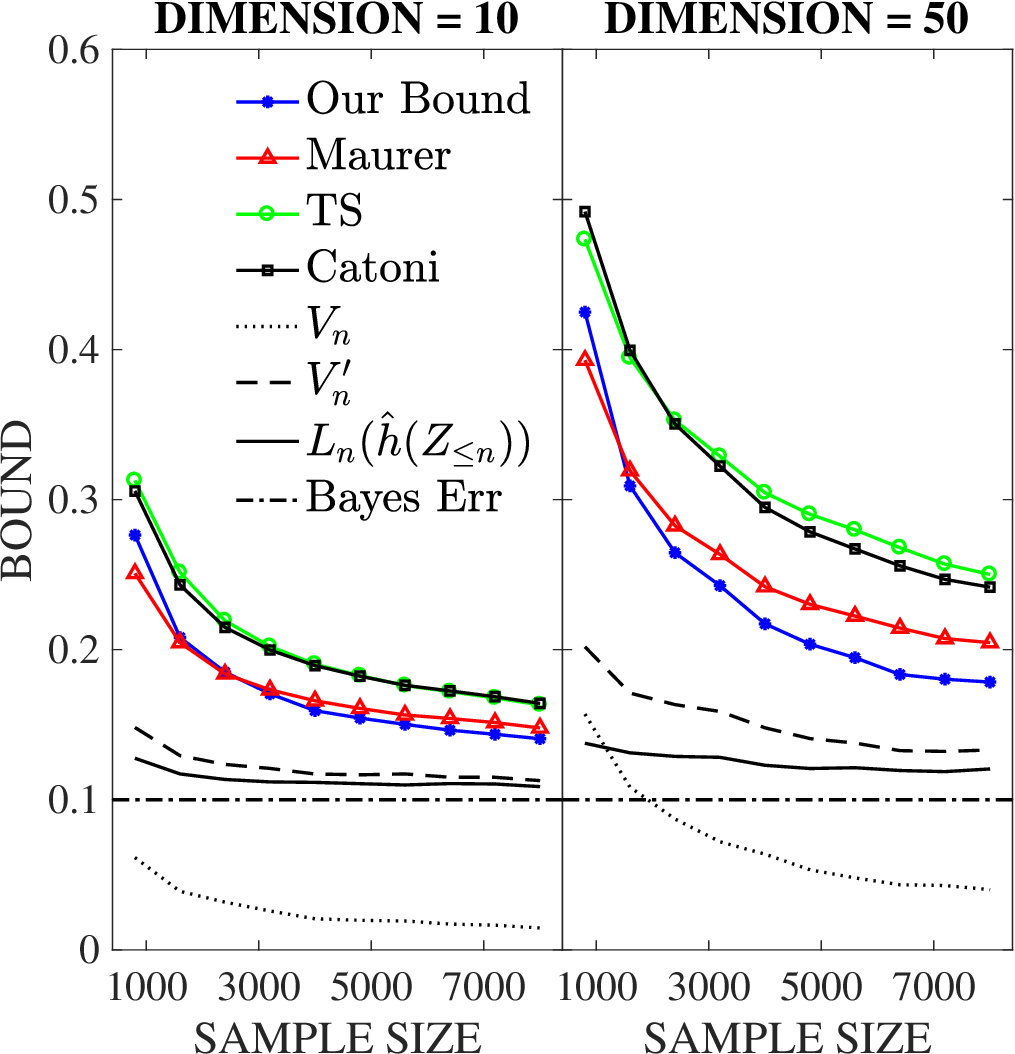} 
			}{%
				\caption{Results for the synthetic data.}%
				\label{fig:synthetic}
			}
			\hspace{-0.cm}
			\capbtabbox{%
				\begin{tabular}{C{1.2cm} C{0.3cm} C{1.1cm} C{0.8cm}  c}
					\hline
					UCI & d   & Test err.    & Our  & Maurer \\
					Dataset	&	  & of $\hat{h}$ & bound & bound \\
					\hline
					Haberman & 3 & 0.272 & 0.521 & \bf{0.411} \\
					
					($n$=244)&  & & & \\
					\hline
					$\text{Breast-C.}$ & 9 & 0.068 & 0.185 & \bf{0.159} \\
					
					($n$=560)&  & & & \\
					\hline
					TicTacToe & 27 & 0.046 & \bf{0.191} & 0.216 \\
					($n$=766) & & & & \\
					\hline
					Banknote  & 4 & 0.058 & \bf{0.125} & 0.136 \\
					($n$=1098) & & & & \\
					\hline
					kr-vs-kp & 73 & 0.044 & \bf{0.108} & 0.165 \\
					($n$=2556) & & & & \\
					\hline
					Spambase & 57 & 0.173 &\bf{0.293} & 0.312 \\
					($n$=3680) & & & & \\		
					\hline
					Mushroom & 116 & 0.002 & {\bf 0.018} & 0.055\\
					($n$=6500) &  & &  &  \\
					\hline 
					Adult & 108 & 0.168 & {\bf 0.195} & 0.234\\
					($n$=24130) &  & &  &  \\
					\hline
				\end{tabular}
			}{%
				\caption{Results for the UCI datasets.}%
				\label{tab:UCI}
			}
		\end{floatrow}
	\end{figure}
	
	\textbf{Setting.} We consider both synthetic and real-world datasets for binary classification, and we evaluate bounds using the 0-1 loss. In particular, the data space $\cZ$ is $\cX \times\cY \coloneqq \mathbb{R}^d \times \{0,1\}$, where $d\in \mathbb{N}$ is the dimension of the feature space.  In this case, the hypothesis set $\cH$ is also $\reals^d$, and the error associated with $h\in \cH$ on a sample $Z=(X,Y)\in \cX\times \cY$ is given by $\ell_h(Z)= |Y - \mathbbm{1}{\{ \phi(h^\top X)> 1/2\}}|$, where $\phi(w)\coloneqq 1/(1+e^{-w}), w\in \reals$. We learn our hypotheses using \emph{regularized logistic regression}; given a sample $S = (Z_{p},\dots,Z_{q})$, with $(p,q)\in\{(1,m),(m+1,n),(1,n)\}$ and $m=n/2$, we compute 
	\begin{align}
	\label{eq:RERM}
	\hat h(S) \coloneqq \argmin_{h \in \cH} \frac{\lambda \|h\|^2 }{2} + \frac{1}{q-p+1}\sum_{i=p}^q  Y_i \cdot \ln \phi(h^\top X_i) + (1-Y_i) \cdot \ln (1-\phi(h^\top X_i)).
	\end{align}
	For $Z_{\leq n} \in \cZ^n$, and $1\leq  i\leq m < j\leq n$, we choose algorithm $Q$ in Theorem \ref{thm:main} such that
	\begin{align}
	\nonumber
	\begin{aligned}
	Q(Z_{>i})  \equiv \delta \left(\hat{h}(Z_{>m})\right)\quad \text{and} \quad Q(Z_{<j}) \equiv \delta(\hat{h}(Z_{\leq m})).
	\end{aligned}
	\end{align}
	Given a sample $S\neq \emptyset$, we set the ``posterior'' $P(S)$ to be a Gaussian centered at $\hat{h}(S)$ with variance $\sigma^2>0$; that is, $P(S) \equiv \cN(\hat{h}(S), \sigma^2 I_d)$. The prior distribution is set to $P_0 \equiv\cN(0, \sigma_0^2 I_d)$, for $\sigma_0>0$. 
	
	\textbf{Parameters.} We set $\delta=0.05$. For all datasets, we use $\lambda=0.01$, and (approximately) solve \eqref{eq:RERM} using the {\sc BFGS} algorithm. For each bound, we pick the $\sigma^2 \in \{1/2,\dots, 1/2^J:J\coloneqq \ceil{\log_2 n}\}$ which minimizes it on the given data (with $n$ instances). In order for the bounds to still hold with probability at least $1-\delta$, we replace $\delta$ on the RHS of each bound by $\delta/ \ceil{\log_2 n}$ (this follows from the application of a union bound). We choose the prior variance such that $\sigma_0^2=1/2$ (this was the best value on average for the bounds we compare against). We choose the grid $\cG$ in Theorem \ref{thm:main} as in \eqref{eq:grid}. Finally, we approximate Gaussian expectations using Monte Carlo sampling. 
	
	\textbf{Synthetic data.} 
	We generate synthetic data for $d=\{10,50\}$ and sample sizes between 800 and 8000. For a given sample size $n$, we 1) draw $X_1, \dots, X_n$ [resp. $\epsilon_1, \dots, \epsilon_n$] identically and independently from the multivariate-Gaussian distribution $\cN(0,I_{d})$ [resp. the Bernoulli distribution $\Bernoulli(0.9)$]; and 2) we set $Y_i =  \mathbbm{1}\{\phi(h_*^\top X_i) > 1/2\} \cdot \epsilon_i$, for $i\in[n]$, where $h_* \in \reals^d$ is the vector constructed from the first $d$ digits of $\pi$. For example, if $d=10$, then $h_* = (3,1,4,1,5,9,2,6,5,3)^\top$. Figure \ref{fig:synthetic} shows the results averaged over 10 independent runs for each sample size.
	
	\textbf{UCI datasets.} For the second experiment, we use several UCI datasets. These are listed in Table~\ref{tab:UCI} (where Breast-C. stands for Breast Cancer). We encode categorical variables in appropriate 0-1 vectors. This effectively increases the dimension of the input space (this is reported as $d$ in Table \ref{tab:UCI}). After removing any rows (\emph{i.e.} instances) containing missing features and performing the encoding, the input data is scaled such that every column has values between -1 and 1. We used a 5-fold train-test split ($n$ in Table \ref{tab:UCI} is the training set size), and the results in Table \ref{tab:UCI} are averages over 5 runs. We only compare with Maurer's bound since other bounds were worse than Maurer's and ours on all datasets.
	
	\textbf{Discussion.}
	As the dimension $d$ of the input space increases, the complexity $\KL(P_n \pipes P_0)$---and thus, all the PAC-Bayes bounds discussed in this paper---get larger. Our bound suffers less from this increase in $d$, since for a large enough sample size $n$, the term $V_n$ is small enough (see Figure \ref{fig:synthetic}) to absorb any increase in the complexity. In fact, for large enough $n$, the irreducible (complexity-free) term involving $V_n'$ in our bound becomes the dominant one. This, combined with the fact that for the 0-1 loss, $V'_n \approx L_n(P_n)$ for large enough $n$ (see Figure \ref{fig:synthetic}), makes our bound tighter than others. 
	
	Adding a regularization term in the objective \eqref{eq:RERM} is important as it stabilizes $\hat{h}(Z_{<m})$ and $\hat{h}(Z_{\geq m})$; a similar effect is achieved with methods like gradient descent as they essentially have a ``built-in'' regularization. For very small sample sizes, the regularization in \eqref{eq:RERM} may not be enough to ensure that $\hat{h}(Z_{<m})$ and $\hat{h}(Z_{\geq m})$ are close to $\hat{h}(Z_{\leq n})$, in which case $V_n$ need not be necessarily small. In particular, this is the case for the Haberman and the breast cancer datasets where the advantage of our bound is not fully leveraged, and Maurer's bound is smaller.
	
	\section{Theoretical Motivation of the Bound}
	\label{sec:bernstein}
	In this section, we study the behavior of our bound  \eqref{eq:highprob} under a Bernstein condition: 
	\begin{definition}{\bf [Bernstein Condition (BC)]}
		The learning problem $(\dist, \ell, \cH)$ satisfies the $(\beta,B)$-Bernstein condition, for $\beta \in[0,1]$ and $B>0$, if for all $h\in \cH$, \begin{align}\nonumber \expect{Z\sim \dist}{\left(\ell_{h}(Z)- \ell_{h_*}(Z)\right)^2} \leq B \cdot \expect{Z \sim \dist}{\ell_{h}(Z)- \ell_{h_*}(Z)}^{\beta},\end{align}
		where $h_* \in \arg \inf_{h\in \cH} \expect{Z\sim \dist}{\ell_h(Z)}$ is a risk minimizer within the closer of $\cH$.
	\end{definition}
	The Bernstein condition \cite{audibert2004pac,bartlett2006convexity,bartlett2006empirical,erven2015fast,koolen2016combining} essentially characterizes the ``easiness'' of the learning problem; it implies that the variance in the excess loss random variable $\ell_h(Z)- \ell_{h_*}(Z)$ gets smaller the closer the risk of hypothesis $h\in \cH$ gets to that of the risk minimizer $h_*$.
	For bounded loss functions, the BC with $\beta=0$ always holds. The BC with $\beta=1$ (the ``easiest'' learning setting) is also known as the {\em Massart noise condition} \cite{massart2006risk}; it holds in our experiment with synthetic data in Section \ref{sec:experiments}, and also, \emph{e.g.}, whenever $\cH$ is convex and $h \mapsto \ell_h(z)$ is exp-concave, for all $z\in \cZ$ \cite{erven2015fast,mehta2017fast}. For more examples of learning settings where a BC holds see \cite[Section 3]{koolen2016combining}.
	
	Our aim in this section is to give an upper-bound on the infimum term involving $V_n$ in \eqref{eq:highprob}, under a BC, in terms of the complexity $\comp_n$ and the excess risks $\risk(P_n)$, $\risk(Q(Z_{>m}))$, and $\risk(Q(Z_{\leq m}))$, where for a distribution $P \in \mathcal{P}(\cH)$, the excess risk is defined by \begin{align} \risk(P) &\coloneqq \expect{h\sim P}{\expect{Z\sim \dist}{\ell_h(Z)}} -\expect{Z\sim \dist}{\ell_{h_*}(Z)}.\nonumber \end{align}
	In the next theorem, we denote $Q_{\leq m}\coloneqq Q(Z_{\leq m})$ and $Q_{>m}\coloneqq Q(Z_{>m})$, for $m\in[n]$. To simplify the presentation further (and for consistency with Section \ref{sec:experiments}), we assume that $Q$ is chosen such that 
	\begin{align}
	Q(Z_{>i})=Q_{>m},\ \text{for $1\leq i\leq m$}, \quad \text{and} \quad Q(Z_{<j}) = Q_{\leq m},\ \text{for $m < j \leq n$}. \label{eq:qchoice}\end{align}
	\begin{theorem}
		\label{thm:vnbound}
		Let $\cG$ and $\pi$ be as in \eqref{eq:grid}, $\delta\in]0,1[$, and $\varepsilon_{\delta,n}=2 \ln \frac{1}{\delta \cdot \pi(\eta)}=2 \ln \frac{|\cG|}{\delta}$, $\eta \in \cG$. If the $(\beta, B)$-Bernstein condition holds with $\beta\in [0,1]$ and $B>0$, then for any learning algorithms $P$ and $Q$ (with $Q$ satisfying \eqref{eq:qchoice}), there exists a $\eC>0$, such that $\forall n\geq 1$ and $m=n/2$, with probability at least $1-\delta$, \begin{align} \frac{1}{\eC} \cdot \inf_{\eta \in \cG } \left\{
		c_{\eta} \cdot  V_n
		+  \frac{\comp_n +\varepsilon_{\delta, n}}{\eta \cdot n} \right\}&  \leq \risk(P_n)+\risk(Q_{\leq m})+\risk(Q_{>m}) \nonumber  \\ &   +  \left(\frac{\comp_n+\varepsilon_{\delta,n}}{n} \right)^{\frac{1}{2-\beta}} + \frac{\comp_n+\varepsilon_{\delta,n}}{n}. \label{eq:boundoninf} \end{align}
	\end{theorem}
	In addition to the ``ESI'' tools provided in Section \ref{sec:esi} and Lemma \ref{lem:bernsand}, the proof of Theorem \ref{thm:vnbound}, presented in Appendix \ref{supp:excessrisk}, also uses an ``ESI version'' of the Bernstein condition due to \citep{koolen2016combining}. 
	
	First note that the only terms in our main bound (\ref{eq:highprob}), other than the infimum on the LHS of \eqref{eq:boundoninf}, are the empirical error $L_n(P_n)$ and a $\tilde{O}(1/\sqrt{n})$-complexity-free term which is typically smaller than $\sqrt{\KL(P_n\pipes P_0)/n}$ (\emph{e.g.} when the dimension of $\cH$ is large enough). The term $\sqrt{\KL(P_n\pipes P_0)/n}$ is often the dominating one in other PAC-Bayesian bounds when $\lim\inf_{n\to \infty} L_n(P_n) >0$.
	
	Now consider the remaining term in our main bound, which matches the infimum term on the LHS of \eqref{eq:boundoninf}, and let us choose algorithm $P$ as per Remark \ref{rmk:prior}, so that $\comp_n = 2 \KL(P_n\pipes P_0)$. Suppose that, with high probability (w.h.p.), $\KL(P_n\pipes P_0)/n$ converges to 0 for $n\to \infty$ (otherwise no PAC-Bayesian bound would converge to 0), then $(\comp_n/n)^{1/(2-\beta)} + \comp_n/n$---essentially the sum of the last two terms on the RHS of \eqref{eq:boundoninf}---converges to 0 at a faster rate than $\sqrt{\KL(P_n\pipes P_0)/n}$ w.h.p. for $\beta>0$, and at equal rate for $\beta=0$. Thus, in light of Theorem \ref{thm:vnbound}, to argue that our bound can be better than others (still when $\lim\inf_{n\to \infty} L_n(P_n) >0$), it remains to show that there exist algorithms $P$ and $Q$ for which the sum of the excess risks on the RHS of \eqref{eq:boundoninf} is smaller than $\sqrt{\KL(P_n\pipes P_0)/n}$.
	
	One choice of estimator with small excess risk is the \emph{Empirical Risk Minimizer} (\ERM{}). When $m=n/2$, if one chooses $Q$ such that it outputs a Dirac around the \ERM{} on a given sample, then under a BC with exponent $\beta$ and for ``parametric'' $\cH$ (such as the $d$-dimensional linear classifiers in Sec.~\ref{sec:experiments}), $\risk(Q_{\leq m})$ and $\risk(Q_{> m})$ are of order $ \tilde{O}\left(n^{-1/(2-\beta)}\right)$ w.h.p. \cite{audibert2004pac,grunwald2016fast}. However, setting $P_n \equiv \delta(\ERM(Z_{\leq n}))$ is not allowed, since otherwise $\KL(P_n\pipes P_0) =\infty$.
	Instead one can choose $P_n$ to be the generalized-Bayes/Gibbs posterior. In this case too, under a BC with exponent $\beta$ and for parametric $\cH$, the excess risk is of order $\tilde{O}\left(n^{-1/(2-\beta)}\right)$ w.h.p. for clever choices of prior $P_0$ \cite{audibert2004pac,grunwald2016fast}.
	
	\section{Detailed Analysis}
	\label{sec:esi}
	We start this section by presenting the convenient ESI notation and use it to present our main technical Lemma \ref{lem:bernsand} (proofs of the ESI results are in Appendix \ref{supp:proofsesi}). We then continue with a  proof of Theorem~\ref{thm:main}. 
	\label{sec:proofsec}
	\begin{definition}{\bf  [ESI (\emph{Exponential Stochastic Inequality, pronounce as:easy}) \citealp{grunwald2016fast,koolen2016combining}]}
		\label{def:esi} Let $\eta>0$, and $X$, $Y$ be any two random variables with joint distribution $\dist$. We define
		\begin{align}\label{eq:esi}
		X \stochleq^\dist_{\eta} \  Y  \ \ \iff \ \ X -Y\stochleq^\dist_{\eta} \  0  \ \ \iff \ \ \Exp_{(X, Y) \sim \dist} \left[e^{\eta (X- Y)} \right] \leq 1.
		\end{align}
	\end{definition}
	Definition \ref{def:esi} can be extended to the case where $\eta=\hat{\eta}$ is also a random variable, in which case the expectation in \eqref{eq:esi} needs to be replaced by the expectation over the joint distribution of ($X$, $Y$, $\hat\eta$). When no ambiguity can arise, we omit $\dist$ from the ESI notation. Besides simplifying notation, ESIs are useful in that they  simultaneously capture ``with high probability'' and ``in expectation'' results: 
	\begin{proposition}{\bf  [ESI Implications]} \label{prop:drop} 
		For fixed $\eta > 0$, if $X \stochleq_{\eta} Y$ then $\E [X] \leq \E[Y]$. 
		For both fixed and random $\hat\eta$, if $X \stochleq_{\hat\eta} Y$, then $\forall \delta \in]0,1[$, $X\leq Y + \frac{\ln \frac{1}{\delta}}{\hat \eta}$, with probability at least $1-\delta$.
	\end{proposition}
	In the next proposition, we present two results concerning transitivity and additive properties of ESI: 
	\begin{proposition}{\bf [ESI Transitivity and Chain Rule]}\label{prop:Trans} 
		(a) Let $Z_1, \dots, Z_n$ be any random variables on $\mathcal{Z}$ (not necessarily independent). If for some $(\gamma_i)_{i \in[n]} \in ]0,+\infty[^n$, $Z_{i} \stochleq_{\gamma_i} 0$, for all $i\in[n]$, then 
		\begin{align}
		\sum_{i=1}^n Z_{i} \stochleq_{\nu_n} 0, \quad \text{where $\nu_n \coloneqq  \left(\sum_{i=1}^n \frac{1}{\gamma_i}\right)^{-1}$} \text{\ \ (so if $\forall i \in [n], \gamma_i = \gamma>0$ then $\nu_n = \gamma/n$)}. \label{eq:wtrn}
		\end{align}
		(b) Suppose now that $Z_1, \ldots, Z_n$ are i.i.d. and let  $X:\cZ \times \bigcup_{i=1}^n \cZ^i \rightarrow \reals$ be any real-valued function. If for some $\eta>0$, $X(Z_i; z_{<i}) \stochleq_\eta 0$, for all $i \in [n]$ and all $z_{<i} \in \cZ^{i-1}$, then $\sum_{i=1}^n X(Z_i; Z_{<i}) \stochleq_{\eta} 0$.
	\end{proposition}
	We now give a basic PAC-Bayesian result for the ESI context: \begin{proposition}{\bf [ESI PAC-Bayes]} \label{prop:donsker}
		Fix $\eta > 0$ and let $\{Y_h: h \in \cH\}$ be any family of random variables such that for all $h \in \cH$, $Y_h \stochleq_{\eta} 0$. Let $\dol_0$ be any distribution on $\cH$ and let $\dol: \bigcup_{i=1}^n \cZ^{i} \rightarrow \mathcal{P}(\cH)$ be a learning algorithm.  We have: 
		\begin{align} \E_{h \sim \dol_n}[Y_{h}] \stochleq_{\eta} \frac{\KL(\dol_n \| \dol_0)}{\eta}, \quad \text{where $\dol_n \coloneqq \dol(Z_{\leq n})$}. \label{eq:PACBayes} \end{align} 
	\end{proposition}
	In many applications (especially for our main result) it is desirable to work with a random (\emph{i.e.} data-dependent) $\eta$ in the ESI inequalities; one can tune $\eta$ after seeing the data. 
	\begin{proposition}{\bf  [ESI from fixed to random $\bm\eta$]}
		\label{prop:randeta}
		Let $\cG$ be a countable subset of $]0,+\infty[$ and let $\pi$ be a prior distribution over $\cG$. Given a countable collection $\{Y_{\eta} : \eta \in \cG \}$ of random variables satisfying $Y_{\eta} \stochleq_{\eta} 0$, for all fixed $\eta \in \cG$, we have, for arbitrary estimator $\estimate{\eta}$ with support on $\cG$,
		\begin{align}
		\label{eq:randesi}
		Y_{\estimate\eta} \stochleq_{\estimate\eta} \frac{- \ln \pi(\estimate\eta)}{\estimate\eta}.
		\end{align}
	\end{proposition}
	\commentout{
		In our main result (Theorem \ref{thm:main}), we present a PAC-Bayesian bound which holds in expectation. To obtain such a result, we need a way to go from an ESI inequality with a random $\eta$ to an inequality which holds in expectation. The next proposition provides us with the means to do this:  
		\begin{proposition}{\bf [From ESI to \Exp]}
			\label{prop:esitoexp}
			Let $\cG$ be a countable subset of $\reals_{>0}$ such that $\sup \cG <\infty$ and $\inf \cG >0$. Given a collection of random variables $\{Y_\eta: \eta \in \cG \}$ and an estimator $\estimate\eta$ with support on $\cG$ satisfying $Y_{\estimate\eta} \stochleq_{\estimate\eta} 0$, we have
			\begin{equation}\label{eq:traina}
			\expect{}{Y_{\estimate\eta}} \leq \expect{}{
				\frac{1 +  \ln \frac{\sup \cG}{\inf\cG} }{\estimate\eta}
			}.
			\end{equation}
		\end{proposition}
	}
	The following key lemma, which is of independent interest, is central to our main result. 
	\begin{lemma}{\bf \ [Key result: {\em un-\/}expected Bernstein]}
		\label{lem:bernsand}
		Let $X\sim \dist$ be a random variable bounded from {\em above}  by $b>0$ almost surely, and let  $\vartheta(u) := 
		(- \ln (1-u) - u)/u^2$. For all $0 < \eta < 1/b$, we have {\bf\em (a)}:
		\begin{align}\label{eq:blade}
		\E[X] - X  \stochleq^{\dist}_{\eta} \ c \cdot X^2,
		\quad \text{ for all $c \geq \eta \cdot  \vartheta(\eta b)$}.
		\end{align} 
		{\bf \em (b)}: The result is tight; for every $c < \eta \cdot \vartheta(\eta b)$, there exists a distribution $\dist$ so that (\ref{eq:blade}) does not hold.
	\end{lemma}
	Lemma \ref{lem:bernsand} is reminiscent of the following slight variation of Bernstein's inequality \citep{CesaBianchiL06}; let $X$ be any random variable bounded from {\em below} by $- b$, and let $\kappa(x) \coloneqq (e^{x} -x - 1)/x^2$. For all $\eta > 0$, we have 
	\begin{equation}\label{eq:ourbernstein}
	\E[X] - X \stochleq_{\eta}\  s \cdot  \E[X^2], \quad \text{ for all $s \geq \eta \cdot \kappa(\eta b)$.}
	\end{equation}
	Note that the un-expected Bernstein Lemma \ref{lem:bernsand}
	has the $X^2$ lifted out of the expectation. In Appendix
	\ref{supp:unexpVSexp}, we prove \eqref{eq:ourbernstein} and
	compare it to standard versions of Bernstein. We also compare
	\eqref{eq:blade} to the related but distinct empirical
	Bernstein inequality due to \cite[Theorem
	4]{maurer2009empirical}.  We now prove part (a) of Lemma
	\ref{lem:bernsand}, which follows easily from the proof of an existing result \cite{fan2015exponential,howard2018uniform}. Part (b) is novel; its proof is postponed to Appendix
	\ref{bernsandproof}.
	\begin{proof}[{\bf Proof of Lemma \ref{lem:bernsand}-Part (a)}]
		\cite{fan2015exponential} (see also \cite{howard2018uniform}) showed in the proof of their lemma 4.1 that 
		\begin{align}
		\exp(\lambda \xi -  \lambda^2 \vartheta(\lambda)\xi^2) \leq 1 + \lambda \xi, \quad \text{for all } \lambda \in[0,1[ \ \text{ and } \  \xi \geq -1. \label{eq:fan}
		\end{align}  
		Letting $\eta = \lambda/b$ and $\xi = -X/b$, \eqref{eq:fan} becomes, 
		\begin{align}
		\exp(-\eta X - \eta^2 \vartheta(\eta b) X^2) \leq 1 - \eta X, \quad \text{for all } \eta \in ]0,1/b[. \label{eq:fannew}
		\end{align} 
		Taking expectation on both sides of \eqref{eq:fannew} and using the fact that $1-\eta \E[X]\leq \exp(-\eta \E[X])$ on the RHS of the resulting inequality, leads to \eqref{eq:blade}.
	\end{proof}
	\begin{proof}[\rmbf{Proof of Theorem \ref{thm:main}}]
		Let $\eta\in]0,1/b[$ and $c_{\eta}\coloneqq \eta \cdot \vartheta(\eta b)$. For $1\leq  i\leq m < j \leq n$, define \begin{align*} X_{h}(Z_i;z_{> i})&\coloneqq \ell_{h}(Z_i) - \expect{\rv{h} \sim Q(z_{>i})}{\ell_{\rv{h}}(Z_i)},  \quad \text{for } z_{>i}\in\cZ^{n-i}, \\
		\tilde{X}_{h}(Z_j;z_{< j})&\coloneqq \ell_{h}(Z_j) - \expect{\rv{h} \sim Q(z_{<j})}{\ell_{\rv{h}}(Z_j)}, \quad \text{for } z_{<j} \in \cZ^{j-1}.
		\end{align*}
		Since $\ell$ is bounded from above by $b$, Lemma \ref{lem:bernsand} implies that for all $ h \in \cH$ and $1\leq  i\leq m < j \leq n$,
		\begin{align*}
		\forall z_{>i} \in \cZ^{n-i}, \quad 
		Y^{\eta}_{h}(Z_i;z_{>i}) &:= \expect{\rv{Z}_i \sim \dist}{X_{h}(\rv{Z}_i;z_{>i})} - X_{h}(Z_i;z_{>i})   - c_\eta \cdot  X_{h}(Z_i;z_{>i})^2 \stochleq_\eta 0, \\
		\forall z_{<j} \in \cZ^{j-1},  \quad 
		\tilde{Y}^{\eta}_{h}(Z_j;z_{<j}) &:= \mathbb{F}_{\rv{Z}_j \sim \dist}[\tilde{X}_{h}(\rv{Z}_j;z_{<j})] - \tilde{X}_{h}(Z_j;z_{<j})   - c_\eta \cdot  \tilde{X}_{h}(Z_j;z_{<j})^2 \stochleq_\eta 0,
		\end{align*}
		Since $Z_{1}, \dots , Z_{n}$ are i.i.d. we can chain the ESIs above using Proposition \ref{prop:Trans}-(b) to get:
		\begin{align}
		S \coloneqq  \sum_{i=1}^m Y^{\eta}_{h}(Z_i; Z_{>i})  \stochleq_\eta 0, \label{eq:untilde} 
		\ \ \ \ \ \tilde S \coloneqq \sum_{j=m+1}^n \tilde{Y}^{\eta}_{h}(Z_j;Z_{<j}) \stochleq_\eta 0.  
		\end{align}
		Applying PAC-Bayes (Proposition \ref{prop:donsker}) to $S$ and $\tilde S$ in \eqref{eq:untilde}
		with priors $\dol(Z_{>m})$ and $\dol(Z_{\leq m})$, respectively, and common posterior $\dol_n = \dol(Z_{\leq n})$ on $\cH$, we get, with ${\KL}_{> m} := \KL(\dol_n\| \dol(Z_{>m}))$ and ${\KL}_{\leq m} := \KL(\dol_n\| \dol(Z_{\leq m}))$:
		\begin{align*}
		\expect{h\sim \dol_n}{   \sum_{i=1}^m Y^{\eta}_{h}(Z_i;Z_{>i})} -\frac{\KL_{> m}}{\eta} \stochleq_\eta 0, 
		\ \ \ \ \  \expect{h \sim \dol_n}{  \sum_{j=m+1}^n \tilde{Y}^\eta_{h}(Z_j;Z_{<j})} - \frac{\KL_{\leq m}}{\eta}  \stochleq_\eta 0.  
		\end{align*}
		We now apply Proposition \ref{prop:Trans}-(a) to chain these two ESIs, which yields
		\begin{align}
		\expect{h \sim \dol_n}{\sum_{i=1}^m Y^{\eta}_{h}(Z_i;Z_{>i}) +\sum_{j=m+1}^n \tilde{Y}^{\eta}_{h}(Z_j;Z_{<j})}  \stochleq_{\frac{\eta}{2}} \frac{\KL(\dol_n \| \dol(Z_{>m})) + \KL(\dol_n\| \dol(Z_{\leq m}) )}{\eta}.  \label{eq:postpac}
		\end{align}
		With the prior $\pi$ on $\cG$, we have for any $\estimate\eta = \estimate\eta(Z_{\leq n}) \in \cG \subset   [1/\sqrt{nb^2},1/b[$ (see Proposition \ref{prop:randeta}),
		\begin{align}
		& \expect{h \sim \dol_n}{\sum_{i=1}^m Y^{\estimate\eta}_{h}(Z_i;Z_{>i}) +\sum_{j=m+1}^n \tilde{Y}^{\estimate\eta}_{h}(Z_j;Z_{<j})}     \stochleq_{\frac{\estimate\eta}{2}} \frac{\comp_n}{\estimate\eta} 
		-\frac{2 \ln \pi(\estimate\eta)}{\estimate\eta} 
		, \text{\ \emph{i.e.},\, } \nonumber \\
		& \hspace{0.7 cm} n\cdot (L(P_n) - L_n(P_n)) 
		\ \ \ \stochleq_{\frac{\estimate\eta}{2}}  n  \cdot c_{\estimate \eta} \cdot V_n +\frac{\comp_n + 2\ln \frac{1}{\pi(\estimate\eta)} }{\estimate\eta} \ \  + 
		\nonumber \\
		& \hspace{1.2 cm} 
		\left[\sum_{i=1}^m \left( \expect{\rv{Z}_i\sim \dist} {  \bar{\ell}_{Q_{>i}}(\rv{Z}_i)} - \bar{\ell}_{Q_{>i}}(Z_i)\right) + \sum_{j=m+1}^n \left( \expect{\rv{Z}_j\sim \dist} {\bar{\ell}_{Q_{<j}}(\rv{Z}_j)} - \bar{\ell}_{Q_{<j}}(Z_j) \right)\right], \label{eq:simplifiedESI}
		\end{align}
		where $\bar{\ell}_{Q_{>i}}(Z_i)\coloneqq \expect{{h} \sim Q(Z_{>i})}{\ell_{{h}}(Z_i)}$ and $\bar{\ell}_{Q_{<j}}(Z_j)\coloneqq \expect{{h} \sim Q(Z_{<j})}{\ell_{{h}}(Z_j)}$. 
		Let $U_n$ denote the quantity between the square brackets in \eqref{eq:simplifiedESI}. Using the un-expected Bernstein Lemma \ref{lem:bernsand}, together with Proposition \ref{eq:randesi}, we get for any estimator $\estimate\nu$ on $\cG$:
		\begin{align}
		U_n \stochleq_{\estimate\nu}\ c_{\estimate\nu}\cdot \left(\sum_{i=1}^m \expect{\rv{h}\sim Q(Z_{>i}) }{ \ell_{\rv{h}}(Z_i)^2} + \sum_{j=m+1}^n \expect{\rv{h}\sim Q(Z_{<j}) }{\ell_{\rv{h}}(Z_j)^2}\right) + \frac{ \ln \frac{1}{\pi(\estimate\nu)} }{\estimate\nu}.\label{eq:secondesi}
		\end{align}
		By chaining \eqref{eq:secondesi} and \eqref{eq:simplifiedESI} using Proposition \ref{prop:Trans}-(a) and dividing by $n$, we get:
		\begin{align} L(\dol_n) \stochleq_{\frac{n \estimate\eta \estimate \nu}{\estimate\eta + 2\estimate\nu}} \emperical(\dol_n) +  
		c_{\estimate\eta} \cdot  V_n
		+  \frac{\comp_n + 2 \ln \frac{1}{\pi(\estimate\eta)}}{\estimate\eta \cdot n} + c_{\estimate\nu} \cdot V'_n + \frac{\ln \frac{1}{\pi(\estimate\nu)}}{\estimate\nu \cdot n}. \label{eq:lastesi}
		\end{align}
		We now apply Proposition~\ref{prop:drop} to \eqref{eq:lastesi} to obtain the following inequality with probability at least $1-\delta$: 
		\begin{align} L(\dol_n) \leq \emperical(\dol_n) +  \left[
		c_{\estimate\eta} \cdot  V_n
		+  \frac{\comp_n +  2 \ln \frac{1}{\pi(\estimate\eta)\cdot \delta}}{\estimate\eta \cdot n} \right]+\left\{ c_{\estimate\nu} \cdot V'_n + \frac{\ln \frac{1}{\pi(\estimate\nu) \cdot \delta }}{\estimate\nu \cdot n }\right\}. \label{eq:protohighprob}
		\end{align}
		Inequality \eqref{eq:highprob} follows after picking $\estimate\nu$ and $\estimate\eta$ to be, respectively, estimators which achieve the infimum over the closer of $\cG$ of the quantities between braces and square brackets in \eqref{eq:protohighprob}.
	\end{proof}
	\section{Conclusion and Future Work}
	The main goal of this paper was to introduce a new PAC-Bayesian bound based on a new proof technique; we also theoretically motivated the bound in terms of a Bernstein condition. The simple experiments we provided are to be considered as a basic sanity check---in future work, we plan to put the bound to real practical use by applying it to deep nets in the style of, \emph{e.g.}, \cite{zhou2018}.
	
	\subsubsection*{Acknowledgments}
	An anonymous referee made some highly informed remarks on our paper, which led us to substantially rewrite the paper and made us understand our own work much better. Part of this work was performed while Zakaria Mhammedi was interning at the Centrum Wiskunde \& Informatica (CWI). This work was also supported by the Australian Research Council and Data61.
	
	\bibliographystyle{plain}
	\bibliography{master}
	
	\clearpage
\begin{appendix}

	\section{Informed Priors}
	\label{app:informedpriors}
	Any bound of the form of \eqref{eq:starter} with $\comp_n = \KL(P_n \pipes P_0)$ can be applied in a way as to replace this $\KL$ term by $\KL(P_n \pipes P(Z_{>m})) + \KL(P_n \pipes P(Z_{\leq m}))$, and thus making use of ``informed priors''. For this, it suffices to apply the bound on each part of the sample, \emph{i.e.} $Z_{>m}$ and $Z_{\leq m}$, and then combine the resulting bounds with a union bound. In fact, suppose that \eqref{eq:starter} holds with $R_n =L_n(P_n)$ and $\mathscr C=0$, and let $\delta \in]0,1[$. Applying the bound on the second part of the sample $Z_{>m}$ with prior $P(Z_{\leq  m})$ and posterior $P_n$, we get, with probability at least $1-\delta$, 
	\begin{align}
	L(P_n) - L_{>m}(P_n)  & \leq \mathscr P 
	\cdot  \sqrt{\frac{L_{>m}(P_n) \cdot \left( \KL(P_n\pipes P(Z_{\leq m})) + \varepsilon_{\delta, n-m} \right) }{n-m}} \nonumber   \\  &\quad   + \mathscr A \cdot \frac{\KL(P_n\pipes P(Z_{\leq m})) + \varepsilon_{\delta, n-m}}{n-m}, \label{eq:secondhalf}
	\end{align}
	where $L_{>m}(P_n)\coloneqq \frac{1}{n-m}  \sum_{j=m+1}^n \E_{h\sim P_n} [\ell_h(Z_j)]$. Similarly, applying the bound on the first half of the sample $Z_{\leq m}$ with prior $P(Z_{>m})$ and posterior $P_n$, we get, with probability at least $1-\delta$, 
	\begin{align}
	L(P_n) - L_{\leq m}(P_n) & \leq \mathscr P 
	\cdot  \sqrt{\frac{  L_{\leq m}(P_n)\cdot  \left(\KL(P_n\pipes P(Z_{> m})) + \varepsilon_{\delta, m}  \right) }{m}} \nonumber   \\  &\quad   + \mathscr A \cdot \frac{\KL(P_n\pipes P(Z_{> m})) + \varepsilon_{\delta, m}}{m}, \label{eq:firsthalf}
	\end{align}
	where $L_{\leq m}(P_n)\coloneqq \frac{1}{m}  \sum_{i=1}^m \E_{h\sim P_n} [\ell_h(Z_i)]$. Let $p\coloneqq  m/n$ and $q\coloneqq(n-m)/n$ (note that $p+q =1$). Applying a union bound and adding $q \times \eqref{eq:secondhalf}$ with $p\times  \eqref{eq:firsthalf}$, yields the bound
	\begin{align}
	L(P_n) -  L_n(P_n) & \leq \mathscr P 
	\cdot  \sqrt{\frac{2 L_n(P_n) \cdot \left( \KL(P_n\pipes P(Z_{> m}))+\KL(P_n\pipes P(Z_{\leq  m}))+ \bar \varepsilon_{\delta, n} \right) }{n}} \nonumber   \\  &\hspace{2cm}   + \mathscr A \cdot \frac{\KL(P_n\pipes P(Z_{> m}))+\KL(P_n\pipes P(Z_{\leq  m})) +  \bar \varepsilon_{\delta, n}}{n}, \label{eq:combined}
	\end{align}
	with probability at least $1-\delta$, where $ \bar \varepsilon_{\delta, n} \coloneqq \varepsilon_{\delta/2, m} + \varepsilon_{\delta/2, n-m}$. To get to \eqref{eq:combined}, we also used the fact that $\sqrt{x} +\sqrt{y} \leq \sqrt{2(x+y)}$, for all $x,y\in \reals_{\geq 0}$.
	
	The above trick does not directly apply to Maurer's bound in \eqref{eq:maurer} (since the dependence on $L(P_n)$ is not linear). Instead, one can use the joint convexity of the binary Kullback-Leibler divergence $\kl$ in its two arguments as in the following proof of Lemma \ref{lem:fwdbwdmaurer}:
	\begin{proof}[{\bf Proof of Lemma \ref{lem:fwdbwdmaurer}}]
		Let $\delta \in]0,1[$. We can write $L_n(P_n)$ as 
		\begin{align}
		L_n(P_n) =   \frac{p}{m} \sum_{i=1}^m \E_{h\sim P_n} [\ell_h(Z_i)] + \frac{q}{n-m}\sum_{j=m+1}^n \E_{h\sim P_n} [\ell_h(Z_j)],
		\end{align}	
		where $p\coloneqq m/n$ and $q\coloneqq (n-m)/n$ (note that $p+q=1$). Let us denote $$L_{\leq m}(P_n) \coloneqq\frac{1}{m}\sum_{i=1}^m \E_{h\sim P_n} [\ell_h(Z_i)]  \ \  \text{and }\   L_{>m}(P_n) \coloneqq \frac{1}{n-m}\sum_{j=m+1}^n \E_{h\sim P_n} [\ell_h(Z_j)]. $$ By the joint convexity of the binary Kullback-Leibler divergence $\kl$ in its two arguments, we have 
		\begin{align}
		\kl(L(P_n)\pipes L_n(P_n)) &=  \kl(p L(P_n) + q L(P_n)\pipes p L_{\leq m}(P_n) + q L_{> m}(P_n)), \nonumber  \\
		& \leq p \cdot \kl(L(P_n)\pipes L_{\leq m}(P_n))  + q  \cdot   \kl(L(P_n)\pipes L_{> m}(P_n)), \nonumber \\
		& \leq  p\cdot \frac{ \KL(P_n\pipes P(Z_{>m}))  +\ln \frac{4 \sqrt{m}}{\delta} }{m}, \nonumber \\
		& \quad +  q\cdot \frac{ \KL(P_n\pipes P(Z_{\leq m}))  +\ln \frac{4 \sqrt{n-m}}{\delta} }{n-m}, \label{eq:maurerlast}
		\end{align}
		with probability at least $1-\delta$, where the last inequality follows by Maurer's bound \eqref{eq:maurer} and the union bound. Substituting the expressions of $p$ and $q$ in \eqref{eq:maurerlast} yields the desired result.
	\end{proof}
	
	\section{Biasing}
	\label{app:biasing}
	A PAC-Bayes bound similar to the one in our Corollary \ref{cor:main} can be obtained from the TS bound. For this, the TS bound must be applied twice, once on each part of the sample (\emph{i.e.} $Z_{\leq m}$ and $Z_{>m}$) to \emph{biased} losses. We demonstrate this in what follows. 
	
	Let $\hat h : \bigcup_{i=1}^n \cZ^i \rightarrow \cH$ be any estimator. The TS bound can be expressed in the form of \eqref{eq:starter} with $\comp_n = \KL(P_n\pipes P_0)$, $\mathscr C=0$, and $R_n = \E_{h\sim P_n} [\var_n[\ell_h(Z)]]$, where $\var_n[X]$ denotes the empirical variance. Applying the TS bound on the second part of the sample $Z_{>m}$ with prior $P_0$ and posterior $P_n$, and with the biased loss $\tilde{\ell}_h(Z)=\ell_h(Z) -\ell_{\hat h (Z_{\leq m})}(Z)$, gives
	\begin{align}
	\tilde{L}(P_n) - \tilde{L}_{>m}(P_n) &\leq \mathscr{P} \cdot \sqrt{\frac{ \E_{h \sim P_n} [\var_{>m}[\tilde{\ell}_h(Z)]] \cdot (\KL(P_n\pipes P_0) + \varepsilon_{\delta, n-m})     }{n-m}} \nonumber \\
	& \hspace{3cm}+ \mathscr A \cdot \frac{\KL(P_n\pipes P_0) + \varepsilon_{\delta,n-m}}{n-m}, \label{eq:TS1}
	\end{align}
	with probability at least $1-\delta$, where $\var_{>m}[X] \coloneqq \frac{1}{n-m}\sum_{i=m+1}^n \left(X_i - \frac{1}{n-m} \sum_{j=m+1}^n X_j  \right)^2$, $\tilde{L}(P_n) \coloneqq \E_{h\sim P_n}[\E_{Z\sim \dist }[\tilde{\ell}_h(Z)]]$, and $\tilde{L}_{>m}(P_n) \coloneqq \frac{1}{n-m}\sum_{j=m+1}^n \E_{h\sim P_n} [\tilde{\ell}_h(Z_j)]$. 
	
	Doing the same on the first part of the sample $Z_{\leq m}$, but now with the loss $\check\ell_h(Z)\coloneqq \ell_h(Z) -\ell_{\hat h (Z_{> m})}(Z)$, yields
	\begin{align}
	\check{L}(P_n) - \check{L}_{\leq m}(P_n) &\leq \mathscr{P} \cdot \sqrt{\frac{ \E_{h \sim P_n} [\var_{\leq m}[\check{\ell}_h(Z)]] \cdot (\KL(P_n\pipes P_0) + \varepsilon_{\delta, m})     }{m}} \nonumber \\
	& \hspace{3cm}+ \mathscr A \cdot \frac{\KL(P_n\pipes P_0) + \varepsilon_{\delta,m}}{m},\label{eq:TS2}
	\end{align}
	with probability at least $1-\delta$, where $\var_{\leq m}[X] \coloneqq \frac{1}{m}\sum_{i=1}^m \left(X_i - \frac{1}{m} \sum_{j=1}^m X_j  \right)^2$, $\check{L}(P_n) \coloneqq \E_{h\sim P_n}[\E_{Z\sim \dist }[\check{\ell}_h(Z)]]$, and $\check{L}_{\leq m}(P_n) \coloneqq \frac{1}{m}\sum_{i=1}^m \E_{h\sim P_n} [\check{\ell}_h(Z_i)]$. 
	
	Two more applications of the TS bound with prior and posterior equal to $P_0$, yields,
	\begin{align}
	&\hspace{-0.5cm} L(\hat h(Z_{\leq m})) - {L}_{> m}(\hat h(Z_{\leq m})) \leq \mathscr{P} \cdot \sqrt{\frac{  \var_{>m}[{\ell}_{\hat h(Z_{\leq m})}(Z)] \cdot \varepsilon_{\delta/2, n-m}     }{n-m}} +  \frac{ \mathscr A \cdot \varepsilon_{\delta/2,n-m}}{n-m},\ \text{and}  \label{eq:TS3} \\
	&\hspace{-0.5cm} L(\hat h(Z_{> m}))-   {L}_{\leq  m}(\hat h(Z_{> m}))   \leq \mathscr{P} \cdot \sqrt{\frac{  \var_{\leq m}[{\ell}_{\hat h(Z_{> m})}(Z)] \cdot \varepsilon_{\delta/2, m}}{m}} +  \frac{ \mathscr A \cdot  \varepsilon_{\delta/2,m}}{m}, \label{eq:TS4}
	\end{align} 
	with probability at least $1-\delta$, where $${L}_{\leq m}(\hat h(Z_{>m})) \coloneqq \frac{1}{m}\sum_{i=1}^m  {\ell}_{\hat h(Z_{>m})}(Z_i) \ \ \text{ and }  \ \ {L}_{> m}(\hat h(Z_{\leq m})) \coloneqq \frac{1}{n-m}\sum_{j=m+1}^n  {\ell}_{\hat h(Z_{\leq m})}(Z_j).$$ Let $p=m/n$ and $q=(n-m)/n$. Applying a union bound and combining \eqref{eq:TS1}-\eqref{eq:TS4}, as $$ q \times ( \eqref{eq:TS1} +\eqref{eq:TS3}) +p \times ( \eqref{eq:TS2} +\eqref{eq:TS4}),$$ 
	yields a bound of the form \eqref{eq:starter} with 
	\begin{align}
	R_n'  &= p \cdot   \var_{\leq m}[{\ell}_{\hat h(Z_{> m})}(Z)] +  q  \cdot \var_{< m}[{\ell}_{\hat h(Z_{\leq m})}(Z)] \leq V_n',\nonumber   \\ 
	R_n &= p \cdot   \E_{h \sim P_n} [\var_{\leq m}[\check{\ell}_h(Z)]]+  q  \cdot \E_{h \sim P_n} [\var_{> m}[\tilde{\ell}_h(Z)]] \leq V_n,\nonumber
	\end{align}
	where $V_n'$ and $V_n$ are as in Corollary \ref{cor:main}.
	
	\paragraph{A Direct Approach.} Though the steps above lead to a bound similar to ours in Corollary \ref{cor:main}, the constants involved may not be optimal. We now re-derive a modification of the TS bound with a $V_n$ term like in Corollary \ref{cor:main}, and with tighter constants. The proof techniques used here are the same as those used in the proof of Theorem \ref{thm:main}. For $\eta \in]0,1/b[$ (where $b>0$ is an upper-bound on the loss $\ell$) and $m\in[2..n]$, define
	\begin{gather} s_{\eta}\coloneqq \eta \cdot \kappa(\eta b), \quad  \text{where} \ \  \kappa(\eta) \coloneqq  (e^{\eta} - \eta -1)/\eta^2,    \\ 
	\text{and } \ \ \tilde c_{\eta} \coloneqq 	 \frac{s_{\eta} m}{2m-2} \left(1+ \frac{ \eta  m}{2m -2} \right)^{-1},  \quad  \lambda(\eta) \coloneqq \frac{\eta \beta(\eta)}{\eta + \beta(\eta)}, \quad \text{where } \beta(\eta) \coloneqq \eta+ \frac{\eta^2 m^2}{2m -2}. \end{gather}
	We assume that $n>2$ is even in the next theorem. We remind the reader of the definitions $$\var_{\leq m}[X] \coloneqq \frac{1}{m}\sum_{i=1}^m \left(X_i - \frac{1}{m} \sum_{j=1}^m X_j  \right)^2 \  \text{and}\   \var_{> m}[X] \coloneqq \frac{1}{n-m}\sum_{i=m+1}^n \left(X_i - \frac{1}{n-m} \sum_{j=m+1}^n X_j  \right)^2.$$ 
	\begin{theorem}{\bf [New PAC-Bayes Empirical Bernstein Bound]}
		\label{thm:TS}
		Let $Z_1,\dots,Z_n$ be i.i.d. with $Z_1 \sim \dist$. Let $m= n/2>1$ and $\pi$ be any distribution with support on a finite or countable grid $\cG \subset  ]0,1/b[$. For any $\delta\in ]0,1[$, learning algorithm $\dol: \bigcup_{i=1}^n \mathcal{Z}^i\rightarrow \mathcal{P}(\cH)$, and estimator $\hat h :  \bigcup_{i=1}^n \mathcal{Z}^i\rightarrow \cH $, we have, 
		\begin{align} L(\dol_n) \leq  \emperical(\dol_n) + \inf_{\eta \in \cG } \left\{
		\tilde{c}_{\eta} \cdot G_n
		+  \frac{\comp_n  +2 \ln \frac{1}{\delta \cdot \pi(\eta)}}{\lambda(\eta) \cdot n} \right\} + \inf_{\nu \in \cG } \left\{ \tilde{c}_{\nu} \cdot G'_n + \frac{\ln \frac{1}{\delta \cdot \pi(\nu)}}{\lambda(\nu) \cdot n }  \right\} , \label{eq:highprobTS}
		\end{align}
		with probability at least $1-\delta$,  where $\comp_n$, $G_n'$, and $G_n$ are the random variables defined by:
		\begin{align} 
		&  \comp_n \coloneqq 
		\KL(\dol_n \| \dol(Z_{\leq m})) + \KL(\dol_n \| \dol(Z_{> m})),  \label{eq:thecompTS} \\
		&G'_n \coloneqq \var_{> m}\left[\ell_{\hat h(Z_{\leq  m})}(Z) \right]+  \var_{\leq m}\left[\ell_{\hat h(Z_{> m})}(Z) \right] , \nonumber
		\\ 
		&   G_n \coloneqq  \expect{h \sim \dol_n}{
			\var_{> m}\left[\ell_h(Z) -  \ell_{\hat h(Z_{\leq  m})}(Z) \right]
			+      \var_{\leq m}\left[\ell_h(Z) -  \ell_{\hat h(Z_{> m})}(Z)   \right]}.
		\end{align}
	\end{theorem}
	Note that since $\var_{\leq m}(X) \leq \sum_{i=1}^m X_i^2/m$ and $\var_{> m}(X) \leq \sum_{i=m+1}^n X_i^2/m$, we have \begin{align}G_n\leq  V_n  \quad \text{and} \quad  G'_n\leq V'_n, \end{align}where $V_n$ and $V_n'$ are defined in \eqref{eq:theVn} and \eqref{eq:irred}, respectively. However, one cannot directly compare $G_n$ to the $V_n$ defined in Theorem \ref{thm:main}, since the latter uses ``online'' posteriors $(Q(Z_{>i}))$ and $Q(Z_{<j})$ which get closer and closer to the posterior $Q(Z_{\leq n})$ based on the full sample.
	
	To prove Theorem \ref{thm:TS}, we need the following self-bounding property of the empirical variance \cite{maurer2009empirical}: 
	\begin{align}
	\label{eq:selfbound}
	m\var[X]  \stochleq_{ \eta}  \frac{m^2}{m-1}  \var_{m}[X] - \frac{\eta m^2}{2m -2}\var[X],
	\end{align}
	for any $\eta >0$ and any bounded random variable $X$, where $\var_m[X]$ is either $\var_{>m}[X]$ or $\var_{\leq m}[X]$ (recall that $m=n/2$). Re-arranging \eqref{eq:selfbound} and dividing by $(1+\eta m/(2m -2))$, leads to 
	\begin{gather}
	\label{eq:selfboundnew}
	m \var[X]  \stochleq_{ \beta(\eta)}  \frac{m^2}{m-1} \cdot \left(1+ \frac{\eta m}{2m -2} \right)^{-1} \var_{m}[X],\\
	\text{where }\ \  \beta(\eta) \coloneqq \eta+ \frac{\eta^2 m}{2m -2}.
	\end{gather}
	\begin{proof}[\rmbf{Proof of Theorem \ref{thm:TS}}]
		Let $\eta\in]0,1/b[$ and $s_{\eta}\coloneqq \eta \cdot \kappa(\eta b)$. We define \begin{align*} X_{h}(Z_i)&\coloneqq \ell_{h}(Z_i) - \ell_{\hat h(Z_{>m}) }(Z_i),  \quad \text{for} \ \ 1 \leq  i \leq m, \\
		\tilde{X}_{h}(Z_j)&\coloneqq \ell_{h}(Z_j) -\ell_{\hat h(Z_{\leq m})}(Z_j),  \quad \text{for} \ \ m <  j \leq n.
		\end{align*}
		Since $\ell$ is bounded from above by $b$, the Bernstein inequality \eqref{eq:ourbernstein} applied to the zero-mean random variables $\expect{\rv{Z}_i \sim \dist}{X_{h}(\rv{Z}_i)} - X_{h}(Z_i), i\in[n],$ implies that for all $h \in \cH$,
		\begin{align*}
		Y^{\eta}_{h}(Z_i) &:= \expect{\rv{Z}_i \sim \dist}{X_{h}(\rv{Z}_i)} - X_{h}(Z_i)   - s_\eta \cdot \var [X_{h}(Z)] \stochleq_\eta 0,  \quad \text{for} \ \ 1 \leq  i \leq m,  \\
		\tilde{Y}^{\eta}_{h}(Z_j) &:= \mathbb{E}_{\rv{Z}_j \sim \dist}[\tilde{X}_{h}(\rv{Z}_j)] - \tilde{X}_{h}(Z_j)   - s_\eta \cdot  \var[ \tilde{X}_{h}(Z)] \stochleq_\eta 0, \quad \text{for} \ \ m <  j \leq n.
		\end{align*}
		Since $Z_{1}, \dots , Z_{n}$ are i.i.d. we can chain the ESIs above using Proposition \ref{prop:Trans}-(b) to get:
		\begin{align}
		S \coloneqq  \sum_{i=1}^m Y^{\eta}_{h}(Z_i)  \stochleq_\eta 0, \label{eq:untildeTS} 
		\ \ \ \ \ \tilde S \coloneqq \sum_{j=m+1}^n \tilde{Y}^{\eta}_{h}(Z_j) \stochleq_\eta 0.  
		\end{align}
		Chaining $S \stochleq_\eta 0$ [resp. $\tilde{S} \stochleq_\eta 0$] and \eqref{eq:selfboundnew} with $\var_m \equiv \var_{\leq  m}$ [resp. $\var_m \equiv \var_{>m}$] using Proposition \ref{prop:Trans}-(a), yields, 
		\begin{align}
		&W^{\eta}_h \stochleq_{\frac{\eta \beta(\eta)}{\eta + \beta(\eta)}}\  0   \quad \text{and} \quad 	\tilde W^{\eta}_h \stochleq_{\frac{\eta \beta(\eta)}{\eta + \beta(\eta)}}\  0,  \quad  \text{where}\label{eq:tildeWs}\\
		&W^{\eta}_h  \coloneqq  \sum_{i=1}^m \left( \expect{\rv{Z}_i \sim \dist}{X_{h}(\rv{Z}_i)} - X_{h}(Z_i) \right)  -   \frac{s_\eta m^2}{m-1} \cdot \left(1+ \frac{\eta m}{2m -2} \right)^{-1} \var_{\leq m} [X_{h}(Z)]  , \\
		&\tilde W^{\eta}_h \coloneqq   \sum_{j=m+1}^n \left( \expect{\rv{Z}_j  \sim \dist}{X_{h}(\rv{Z}_j)} - X_{h}(Z_j) \right)  -  \frac{s_\eta m^2}{m-1} \cdot \left(1+ \frac{\eta m}{2m -2} \right)^{-1} \var_{> m} [X_{h}(Z)] .
		\end{align}
		Let $\lambda(\eta)\coloneqq \eta \beta(\eta)/(\beta(\eta)+\eta)$. Applying PAC-Bayes (Proposition \ref{prop:donsker}) to $W^{\eta}_h \stochleq_{ \lambda(\eta)} 0$ and $\tilde W^{\eta}_h\stochleq_{ \lambda(\eta)} 0$ in \eqref{eq:tildeWs},
		with priors $\dol(Z_{>m})$ and $\dol(Z_{\leq m})$, respectively, and posterior $\dol_n = \dol(Z_{\leq n})$ on $\cH$, we get:
		\begin{align*}
		\expect{h\sim \dol_n}{W^{\eta}_h} -\frac{ \KL(\dol_n\| \dol(Z_{>m}))}{\lambda(\eta)} \stochleq_{\lambda(\eta)} 0, 
		\ \ \ \ \  \expect{h \sim \dol_n}{   \tilde W^{\eta}_h} - \frac{\KL(\dol_n\| \dol(Z_{\leq m}))}{\lambda(\eta)}  \stochleq_{\lambda(\eta)} 0.  
		\end{align*}
		We now apply Proposition \ref{prop:Trans}-(a) to chain these two ESIs, which yields
		\begin{align}
		\expect{h \sim \dol_n}{W^{\eta}_h + \tilde W^{\eta}_h}  \stochleq_{\frac{\lambda(\eta)}{2}} \frac{\KL(\dol_n \| \dol(Z_{>m})) + \KL(\dol_n\| \dol(Z_{\leq m}) )}{\lambda(\eta)}.  \label{eq:postpacTS}
		\end{align}
		With the discrete prior $\pi$ on $\cG$, we have for any $\estimate\eta = \estimate\eta(Z_{\leq n}) \in \cG \subset 1/b \cdot [1/\sqrt{n},1[$ (see Proposition \ref{prop:randeta}),
		\begin{align}
		& \expect{h \sim \dol_n}{W^{\eta}_h + \tilde W^{\eta}_h}     \stochleq_{\frac{\lambda(\hat\eta)}{2}} \frac{\comp_n}{\lambda(\estimate\eta)} 
		-\frac{2\ln \pi(\estimate\eta)}{\lambda(\estimate\eta)} 
		, \text{\ \emph{i.e.},\, } \nonumber \\
		& \hspace{0.7 cm} n\cdot (L(P_n) - L_n(P_n)) 
		\ \ \ \stochleq_{\frac{\lambda(\estimate\eta)}{2}}  n  \cdot \tilde{c}_{\estimate \eta} \cdot G_n +\frac{\comp_n +2\ln \frac{1}{\pi(\estimate\eta)} }{\lambda(\estimate\eta)} \ \  + 
		\nonumber \\
		& \hspace{0cm}
		\left[\sum_{i=1}^m \left( \expect{\rv{Z}_i\sim \dist} {  \ell_{\hat h_{> m}}(\rv{Z}_i)} - \ell_{\hat h_{> m}}(Z_i)\right) + \sum_{j=m+1}^n \left( \expect{\rv{Z}_j\sim \dist} {{\ell}_{\hat h_{\leq m}}(\rv{Z}_j)} - \ell_{\hat h_{\leq  m}}(Z_j) \right)\right], \label{eq:simplifiedESITS}
		\end{align}
		where $\hat h_{>  m} \coloneqq \hat h(Z_{>m})$ and $\hat h_{\leq  m} \coloneqq \hat h( Z_{\leq m})$. Let $U_n$ denote the quantity between the square brackets in \eqref{eq:simplifiedESITS}. Using the Bernstein inequality in \eqref{eq:ourbernstein} chained with \eqref{eq:selfboundnew}, and Proposition \ref{eq:randesi}, we get for any estimator $\estimate\nu$ on $\cG$:
		\begin{align}
		U_n \stochleq_{\lambda(\estimate\nu)} \ n \cdot  \tilde{c}_{\estimate\nu}\cdot \left( \var_{\leq m} [\ell_{\hat{h}(Z_{>m})}(Z)] + \var_{> m} [\ell_{ \hat h(Z_{\leq m})}(Z)] \right) + \frac{ \ln \frac{1}{\pi(\estimate\nu)} }{\lambda(\estimate\nu)}.\label{eq:secondesiTS}
		\end{align}
		By chaining \eqref{eq:simplifiedESITS} and \eqref{eq:secondesiTS} using Proposition \ref{prop:Trans}-(a), dividing by $n$
		, we get:
		\begin{align} L(\dol_n) \stochleq_{\frac{n \lambda(\estimate\eta) \lambda(\estimate \nu)}{ \lambda(\estimate\eta) + 2\lambda(\estimate\nu)}} \emperical(\dol_n) +  
		\tilde{c}_{\estimate\eta} \cdot  G_n
		+  \frac{\comp_n + 2 \ln \frac{1}{\pi(\estimate\eta)}}{\lambda(\estimate\eta) \cdot n} + \tilde{c}_{\estimate\nu} \cdot G_n' + \frac{\ln \frac{1}{\pi(\estimate\nu)}}{\lambda(\estimate\nu) \cdot n}. \label{eq:lastesiTS}
		\end{align}
		We now apply Proposition~\ref{prop:drop} to \eqref{eq:lastesiTS} to obtain the following inequality with probability at least $1-\delta$: 
		\begin{align} L(\dol_n) \leq \emperical(\dol_n) +  \left[
		\tilde{c}_{\estimate\eta} \cdot  G_n
		+  \frac{\comp_n +  2 \ln \frac{1}{\pi(\estimate\eta)\cdot \delta}}{\lambda(\estimate\eta) \cdot n} \right]+\left\{ \tilde{c}_{\estimate\nu} \cdot G_n' + \frac{\ln \frac{1}{\pi(\estimate\nu) \cdot \delta }}{\lambda(\estimate\nu) \cdot n }\right\}. \label{eq:protohighprobTS}
		\end{align}
		Inequality \eqref{eq:highprob} follows after picking $\estimate\nu$ and $\estimate\eta$ to be, respectively, estimators which achieve the infimum over the closer of $\cG$ of the quantities between braces and square brackets in \eqref{eq:protohighprobTS}.
	\end{proof}

	\section{Proof of Lemma \ref{lem:irreduce}}
	\label{app:otherproofs}
	\begin{proof}
		Throughout this proof, we denote $\hat h_{>m} \coloneqq \hat h(Z_{>m})$ and $\hat h_{\leq m} \coloneqq \hat h(Z_{\geq m})$. Let $\delta \in]0,1[$. Since the sample $Z_{\leq m}$ is independent of $Z_{>m}$, we have 
		\begin{align}
		\frac{2}{n} \sum_{i=1}^m \ell_{\hat h_{>m}}(Z_i)^2 & =  \var_{\leq m}[\ell_{\hat h_{>m}}(Z)] + \left(\frac{1}{m} \sum_{i=1}^m \ell_{\hat h_{>m}}(Z_i) \right)^2. \label{eq:leftpart}
		\end{align}
		On the other hand, from \citep[Theorem 10]{maurer2009empirical}, we have 
		\begin{align}
		\var_{\leq m}[\ell_{\hat h_{>m}}(Z)]& \leq \frac{2(m-1)}{m} \var [\ell_{\hat h_{>m}}(Z)] + \frac{8\ln \frac{1}{\delta}}{n}, \nonumber   \\ \label{eq:mau} 
		& \stackrel{|\ell|\leq 1}\leq \frac{2(m-1)}{m}  L(\hat h_{>m}) + \frac{8\ln \frac{1}{\delta}}{n},
		\end{align}
		with probability at least $1-\delta$. By Hoeffding's inequality, we also have 
		\begin{align}
		\left(\frac{1}{m} \sum_{i=1}^m \ell_{\hat h_{>m}}(Z_i)\right)^2 &\leq 2 L(\hat h_{>m})^2 + \frac{8 \ln \frac{1}{\delta}}{n},\nonumber  \\  
		& \stackrel{|\ell|\leq 1}{\leq}  2 L(\hat h_{>m}) + \frac{8 \ln \frac{1}{\delta}}{n},   \label{eq:hoeff}
		\end{align}
		with probability at least $1-\delta$. Combining \eqref{eq:leftpart}, \eqref{eq:mau}, and \eqref{eq:hoeff} together using a union bound, yields
		\begin{align}
		\frac{2}{n} \sum_{i=1}^m \ell_{\hat h_{>m}}(Z_i)^2 \leq  \frac{4(n -1)}{n} L(\hat h_{>m}) +\frac{16 \ln \frac{2}{\delta}}{n}, \label{eq:finalleafthand}
		\end{align}
		with probability at least $1-\delta$. Applying the same argument on the second part of the sample $Z_{>m}$, yields
		\begin{align}
		\frac{2}{n} \sum_{j=m+1}^n \ell_{\hat h_{\leq m}}(Z_i)^2 \leq  \frac{4(n -1)}{n} L(\hat h_{\leq m}) +\frac{16 \ln \frac{2}{\delta}}{n}, \label{eq:finalrighhand}
		\end{align}
		with probability at least $1-\delta$. Applying a union bound, and adding together \eqref{eq:finalleafthand} and \eqref{eq:finalrighhand} then dividing by 2, yields,
		\begin{align}
		R_n' &\leq  \frac{2(n -1)}{n} \left(L(\hat h_{\leq m})+L(\hat h_{> m})\right) +\frac{16 \ln \frac{4}{\delta}}{n},  \nonumber \\
		& \leq  2 \left(L(\hat h_{\leq m})+L(\hat h_{> m})\right) +\frac{16 \ln \frac{4}{\delta}}{n},
		\label{eq:rnprimebound}
		\end{align}
		with probability at least $1-\delta$. Diving \eqref{eq:rnprimebound} by $n$ and applying the square-root yields the desired result.
	\end{proof}

	\section{Proofs for Section \ref{sec:esi}}
	\label{supp:proofsesi}
	\begin{proof}[\rmbf{Proof of Proposition \ref{prop:drop}}]
		\label{ESIconvproof} Let $Z=X-Y$.
		For fixed $\eta$, Jensen's inequality yields $\E[Z]\leq 0$. 
		For $\eta=\hat\eta$ that is either fixed or itself a random variable, applying Markov's inequality to the random variable $e^{-\hat\eta Z}$ yields $Z\leq \frac{\ln \frac{1}{\delta}}{\hat\eta}$, with probability at least $1-\delta$, for any $\delta\in]0,1[$.
	\end{proof}
	\begin{proof}[\rmbf{Proof of Proposition \ref{prop:Trans}}]
		\label{Transproof}{\bf [Part (a)]}
		Fix $(\gamma_i)_{i\in [n]} \in ]0,+\infty[^n$, and let $\nu_ j \coloneqq  \left(\sum_{i=1}^j \frac{1}{\gamma_i}\right)^{-1}$, for $j\in[n]$. We proceed by induction to show that $\forall j\in[n],\;\sum_{i=1}^j Z_i \stochleq_{\nu_j} 0$. The result holds trivially for $j=1$, since $\nu_1=\gamma_1$. Suppose that 
		\begin{align}
		\sum_{i=1}^j Z_i \stochleq_{\nu_j} 0, \label{claim}
		\end{align}
		for some $1\leq j<n$. We now show that \eqref{claim} holds for $j+1$; we have,
		\begin{eqnarray}
		\Exp  \left[e^{\frac{\nu_j \gamma_{j+1}}{\nu_j +\gamma_{j+1}} \left( \sum_{i=1}^j Z_i + Z_{j+1} \right) } \right]& =& \Exp\left[e^{\frac{\nu_j \gamma_{j+1}}{\nu_j +\gamma_{j+1}}  \sum_{i=1}^j Z_i  +\frac{\nu_j \gamma_{j+1}}{\nu_j +\gamma_{j+1}}  Z_{j+1} }\right], \nonumber  \\ 
		&\stackrel{\text{Jensen}}{\leq} & \tfrac{\gamma_{j+1}}{\nu_j +\gamma_{j+1}} \Exp\left[e^{\nu_j \sum_{i=1}^j Z_i  }\right] + \tfrac{\nu_j}{\nu_j +\gamma_{j+1}} \Exp\left[e^{\gamma_{j+1}   Z_{j+1} }\right], \nonumber \\ 
		& \stackrel{\text{using }\eqref{claim}}{\leq} &1. \nonumber
		\end{eqnarray}
		Thus the result holds for $j+1$, since $\nu_{j+1}=\frac{\nu_j \gamma_{j+1}}{\nu_j+ \gamma_{j+1}}$. This establishes \eqref{eq:wtrn}. 
		
		{\bf [Part (b)]} This is a special case of  \citep[Lemma 6]{koolen2016combining}, who treat the general case with non-i.i.d. distributions.
	\end{proof}
	
	\begin{proof}[\rmbf{Proof of Proposition \ref{prop:donsker}}]
		Let $\rho(h) = (d\dol_n/ d\dol_0)(h)$ be the density of $h\in \cH$ relative to the prior measure $\dol_0$. We then have $\KL(\dol_n \| \dol_0) = \E_{h \sim \dol_n} [\ln \rho(h)]$. We can now write:
		\begin{align*}
		\E
		\left[ e^{\eta \E_{h \sim \dol_n}[Y_{h}] - \KL(\dol_n \| \dol_0)} \right] &= \E_{} 
		\left[ e^{\eta \E_{h \sim \dol_n}[Y_{h} - \ln \rho(h)]} \right],  \\ 
		&\leq\E_{} \left[ \E_{h \sim \dol_n}
		\left[ e^{\eta (Y_{h} - \ln \rho(h))} \right] \right],  \quad (\text{Jensen's Inequality}) \\
		&=  \E_{} \left[\E_{h \sim \dol_n} \left[\frac{d \dol_0}{d \dol_{n}} \cdot
		e^{\eta Y_{h}} \right] \right],  \\
		&= \E_{} \left[ \E_{h \sim \dol_0} \left[
		e^{\eta Y_{h}} \right]\right], \\ &  = \E_{h \sim \dol_0} \left[ \E_{}  \left[
		e^{\eta Y_{h}} \right] \right],  \quad \quad \quad  \quad \quad (\text{Tonelli's Theorem})  \\ & = 1,
		\end{align*}
		where the final equality follows from our assumption that $Y_h \stochleq_{\eta} 0$, for all $h \in \cH$.
	\end{proof}
	
	\begin{proof}[\rmbf{Proof of Proposition \ref{prop:randeta}}]
		Since $Y_{\eta} \stochleq_{\eta} 0$, for $\eta \in \cG$, we have in particular: \begin{align} 1 \geq \expect{}{\sum_{\eta \in \cG} \pi(\eta) e^{\eta Y_{\eta}}}  \geq  \expect{}{\pi(\estimate\eta) e^{ \estimate\eta Y_{\estimate\eta}}}, \label{eq:protorandesi} \end{align}
		where the right-most inequality follows from the fact that the expectation of a countable sum of positive random variable is greater than the expectation of a single element in the sum. Rearranging \eqref{eq:protorandesi} gives \eqref{eq:randesi}.
	\end{proof}
	
	\section{Proof of Theorem \ref{thm:vnbound}}
	\label{supp:excessrisk}
	In what follows, for $h\in \cH$, we denote $X_h(Z)\coloneqq \ell_h(Z)- \ell_{h_*}(Z)$ the excess loss random variable, where $h_*$ is the risk minimizer within $\cH$. Let $$\rho(\eta)\coloneqq \frac{1}{\eta}  \ln \expect{Z\sim \dist}{e^{-\eta X_h(Z)}}$$ be its \emph{normalized cumulant generating function}.
	We need the following useful lemmas:
	\begin{lemma}{\rmbf{\citep{koolen2016combining}}}
		\label{lem:koolen1}
		Let $h\in \cH$ and $X_h$ be as above. Then, for all $\eta\geq0$,
		\begin{align}
		\alpha_\eta \cdot X_h(Z)^2 - X_h(Z) \stochleq_{\eta} \rho(2\eta) + \alpha_\eta \cdot \rho(2\eta)^2, \quad \text{where } \alpha_\eta \coloneqq \frac{\eta}{1+\sqrt{1+4 \eta^2}}. \nonumber 
		\end{align}
	\end{lemma}
	\begin{lemma}{\rmbf{\citep{koolen2016combining}}} 
		\label{lem:koolen2}
		Let $b>0$, and suppose that $X_h \in[-b,b]$ almost surely, for all $h\in \cH$. If the $(\beta,B)$-Bernstein condition holds with $\beta \in[0,1]$ and $B>0$, then 
		\begin{align}
		\rho(\eta) \leq (B \eta)^{\frac{1}{1-\beta}}, \quad \text{for all $\eta \in]0,1/b]$}. \nonumber
		\end{align}
	\end{lemma}
	\begin{lemma}{\rmbf{\citep{CesaBianchiL06}}}
		\label{lem:cesa}
		Let $b>0$, and suppose that $X_h \in[-b,b]$ almost surely, for all $h\in \cH$. Then
		\begin{align}
		\rho(\eta) \leq \frac{\eta b^2}{2}, \quad \text{for all $\eta \in \reals$}.\nonumber
		\end{align}
	\end{lemma}
	\begin{proof}[\rmbf{Proof of Theorem \ref{thm:vnbound}}]First we apply the following inequality \begin{align}
		(a-d)^2 \leq 2(a-c)^2 +2 (d-c)^2 \label{eq:addandsubstract}\end{align} 
		which holds for all $a,c,d\in \reals$ to upper bound $V_n$. Let's focus on the first term in the expression of $V_n$, which we denote $V_n^{\text{left}}$: that is,
		\begin{align}
		V_n^{\text{left}} &\coloneqq  \expect{h \sim P_n}{\frac{1}{n} \sum_{i=1}^{m} \left(\ell_{h}(Z_{i}) - \expect{\rv{h} \sim Q(Z_{>i})}{\ell_{\rv{h}} (Z_i)}\right)^2 }. \label{eq:vleft}
		\end{align}
		Letting $X_h(Z)\coloneqq \ell_h(Z)- \ell_{h_*}(Z)$ and applying \eqref{eq:addandsubstract} with $a=\ell_h(Z_i)$, $c=\ell_{h_*}(Z_i)$, and $d = \expect{\rv{h} \sim Q(Z_{>i})}{\ell_{\rv{h}} (Z_i)} \stackrel{*}{=} \expect{\rv{h} \sim Q(Z_{>m})}{\ell_{\rv{h}} (Z_i)}$ (where $\stackrel{*}{=}$ is due to our assumption on $Q$), we get:
		\begin{align}
		V^{\text{left}}_n &\leq  \expect{h \sim P_n}{\frac{2}{n} \sum_{i=1}^{m} X_h(Z_i)^2 } + \frac{2}{n} \sum_{i=1}^{m} \left( \expect{\rv{h} \sim Q(Z_{>m})}{\ell_{\rv{h}} (Z_i)}- \ell_{h_*}(Z_{i}) \right)^2, \nonumber \\
		&\leq \expect{h \sim P_n}{\frac{2}{n} \sum_{i=1}^{m} X_h(Z_i)^2} + \expect{h \sim Q(Z_{>m})}{\frac{2}{n} \sum_{i=1}^{m} X_h(Z_i)^2}. \quad \text{(by Jensen's Inequality)}\label{eq:upper}
		\end{align}
		Let $i\in[m]$, $h \in \cH$, and $\eta \in ]0,1/b[$. Under the $(\beta, B)$-Bernstein condition, Lemmas \ref{lem:koolen1}-\ref{lem:cesa} imply,
		\begin{align}
		\alpha_\eta \cdot X_h(Z_i)^2 \stochleq_{\eta} X_h(Z_i)+ \left(1+\tfrac{b}{2}\right) \left(2B \eta\right)^{\frac{1}{1-\beta}},  \label{eq:koolen}
		\end{align}
		where $\alpha_\eta \coloneqq \eta/(1+\sqrt{1+4\eta^2})$. Now, due to the Bernstein inequality \eqref{eq:ourbernstein}, we have 
		\begin{align}
		X_h(Z_i)& \stochleq_{\eta} \expect{Z'_i \sim \dist}{X_h(Z'_i)} + s_{\eta} \cdot \expect{Z'_i \sim \dist}{X_h(Z'_i)^2}, \quad \where s_{\eta} \coloneqq \eta \cdot \kappa(\eta b),\nonumber \\
		& \stochleq \expect{Z'_i \sim \dist}{X_h(Z'_i)} + s_{\eta}\cdot \expect{Z'_i \sim \dist}{X_h(Z'_i)}^{\beta},  \quad  \text{(by the Bernstein condition)} \nonumber \\ 
		&\stochleq_{\eta} 2 \expect{Z'_i \sim \dist}{X_h(Z'_i)} + a_{\beta}^{\frac{\beta}{1-\beta}} \cdot    (s_{\eta})^{\frac{1}{1-\beta}}, \quad \where a_{\beta}\coloneqq (1-\beta)^{1-\beta}\beta^{\beta}. \label{eq:individual}
		\end{align}
		The last inequality follows by the fact that $z^{\beta} = a_{\beta} \cdot \inf_{\nu >0} \{z/\nu + \nu^{\frac{\beta}{1-\beta}} \} $, for $z\geq 0$ (in our case, we set $\nu = a_{\beta} \cdot s_{\eta}$ to get to \eqref{eq:individual}). By chaining \eqref{eq:koolen} with \eqref{eq:individual} using Proposition \ref{prop:Trans}-(a), we get:
		\begin{align}
		\alpha_\eta \cdot X_h(Z_i)^2 &\stochleq_{\frac{\eta}{2}}\  2 \expect{Z'_i \sim \dist}{X_h(Z'_i)} + a_{\beta}^{\frac{\beta}{1-\beta}} \cdot    (s_{\eta})^{\frac{1}{1-\beta}} + \left(1+\tfrac{b}{2}\right) (2B\eta)^{\frac{1}{1-\beta}}. \nonumber \\
		& \stochleq_{\frac{\eta}{2}}\  2 \expect{Z'_i \sim \dist}{X_h(Z'_i)} + \mathscr{P} \cdot \eta^{\frac{1}{1-\beta}}, \ \text{with}\ \ \mathscr{P}\coloneqq a_{\beta}^{\frac{\beta}{1-\beta}} + \left(1+\tfrac{b}{2}\right)(2B)^{\frac{1}{1-\beta}},   \label{eq:chained}
		\end{align}
		where in the last inequality we used $\kappa(1)\leq 1$. Since \eqref{eq:chained} holds for all $h\in \cH$, it still holds in expectation over $\cH$ with respect to the distribution $Q(Z_{>m})$ (recall that $i\leq m$); 
		\begin{align}
		\label{eq:expectedchained}
		\alpha_\eta\cdot  \expect{h\sim Q(Z_{>m})}{X_h(Z_i)^2} \stochleq_{\frac{\eta}{2}}\  2 \expect{h\sim Q(Z_{>m}) }{\expect{Z'_i \sim \dist}{X_h(Z'_i)}} + \mathscr{P} \cdot \eta^{\frac{1}{1-\beta}}.
		\end{align}
		Since the samples $Z{_{\leq n}}$ are i.i.d, we have $\expect{Z_i \sim \dist}{\ell_h(Z_i)} =\expect{Z_j \sim \dist}{\ell_h(Z_j)}$, for all $i,j\in[m]$. Thus, after summing \eqref{eq:chained} and \eqref{eq:expectedchained}, for $i=1,\dots,m$, using Proposition \ref{prop:Trans}-(b) and dividing by $n$, we get 
		\begin{align}
		\label{eq:chainsum}
		\frac{\alpha_\eta}{n} \sum_{i=1}^m  X_h(Z_i)^2 &\stochleq_{\frac{n \cdot \eta}{2}}\    \expect{Z \sim \dist}{X_h(Z)} + \frac{\mathscr{P}}{2} \cdot \eta^{\frac{1}{1-\beta}},   \\ 
		\expect{h\sim Q(Z_{>m}) }{\frac{\alpha_\eta}{n} \sum_{i=1}^m  X_h(Z_i)^2}& \stochleq_{\frac{n \cdot \eta}{2}}\ \expect{h\sim Q(Z_{>m}) }{ \expect{Z \sim \dist}{X_h(Z)}} + \frac{\mathscr{P}}{2} \cdot \eta^{\frac{1}{1-\beta}}. \quad \text{($m=n/2$)} \label{eq:expectedchainsum}
		\end{align}
		Now we apply PAC-Bayes (Proposition \ref{prop:donsker}) to \eqref{eq:chainsum}, with prior $P(Z_{>m})$ and posterior $P_n$, and obtain:
		\begin{align}
		\label{eq:pacchainsum}
		\expect{h\sim P_n }{\frac{\alpha_\eta}{n} \sum_{i=1}^m  X_h(Z_i)^2} \stochleq_{\frac{n \cdot \eta}{2}}\   \expect{h\sim P_n }{ \expect{Z \sim \dist}{X_h(Z)}} +  \frac{\mathscr{P}}{2} \cdot  \eta^{\frac{1}{1-\beta}} + \frac{2\KL(P_n\pipes P(Z_{>m}))}{\eta\cdot n}. 
		\end{align}
		Note that the upper-bound on $V^{\text{left}}_n$ in \eqref{eq:upper} is the sum of the left-hand sides of \eqref{eq:expectedchainsum} and \eqref{eq:pacchainsum} divided by $\alpha_\eta/2$. From now on, we restrict $\eta$ to the range $]0,1/(2b)[$ and define $$\mathscr{A}_{\eta} \coloneqq \tfrac{2c_{\eta}}{\alpha_\eta}\leq 2\vartheta\left(\tfrac{1}{2}\right)\cdot \left(1+\sqrt{1+\tfrac{1}{b^2}}\right) \eqqcolon \mathscr{A}, \ \ \quad   \eta \in\left]0,\tfrac{1}{2b}\right[.$$ 
		Chaining \eqref{eq:expectedchainsum} and \eqref{eq:pacchainsum} using Proposition~\ref{prop:Trans}-(a) and multiplying throughout by $\mathscr{A}_{\eta}$, yields
		\begin{align}
		c_{\eta} \cdot V_n^{\text{left}} \stochleq_{\frac{ n \eta }{4 \mathscr{A}_{\eta}}} \mathscr{A}\cdot\left(\risk(P_n)+\risk(Q(Z_{>m}))\right) + \mathscr{P} \mathscr{A} \eta^{\frac{1}{1-\beta}} +\frac{2\mathscr{A}\cdot \KL(P_n\pipes P(Z_{>m}))}{\eta \cdot n}. \label{eq:protononrandesi}
		\end{align}
		By a symmetric argument, a version of \eqref{eq:protononrandesi}, with $Q(Z_{>m})$ [resp. $P(Z_{>m})$] replaced by $Q(Z_{\leq m})$ [resp. $P(Z_{\leq m})$], holds for $ V^{\text{right}}_n \coloneqq V_n - V^{\text{left}}_n$. Using Proposition \ref{prop:Trans}-(a) again, to chain the ESI inequalities of $c_{\eta} \cdot V_n^{\text{left}}$ and $c_{\eta} \cdot V_n^{\text{right}}$, we obtain:
		\begin{align}
		c_{\eta} \cdot V_n \stochleq_{\frac{ n \eta }{8\mathscr{A}_{\eta}}} \mathscr{A} \cdot\left(2\risk(P_n)+\risk(Q_{\leq m})+\risk(Q_{>m})\right) +2 \mathscr{P}\mathscr{A} \eta^{\frac{1}{1-\beta}} +\frac{2\mathscr{A} \cdot \comp_n}{\eta \cdot n}, \label{eq:nonrandesi}
		\end{align}
		where $Q_{>m}\coloneqq Q(Z_{>m})$ and $Q_{\leq m} \coloneqq Q(Z_{\leq m})$.
		Let $\delta \in ]0,1[$, and $\pi$ and $\cG$ be as in \eqref{eq:grid}. Applying Proposition \ref{prop:randeta} to \eqref{eq:nonrandesi} to obtain the corresponding ESI inequality with a random estimator $\estimate{\eta}= \estimate{\eta}(Z_{\leq n})$ with support on $\cG$, and then applying Proposition \ref{prop:drop}, we get, with probability at least $1-\delta$, 
		\begin{align}
		c_{\estimate\eta} \cdot V_n \leq  \mathscr{A}\cdot\left(2\risk(P_n)+\risk(Q_{\leq m})+\risk(Q_{>m})\right)  + 2\mathscr{P}\mathscr{A} \estimate\eta^{\frac{1}{1-\beta}} +\frac{2\mathscr{A} \cdot \comp_n+ 8  \eA \ln \frac{|\cG|}{\delta }}{\estimate\eta \cdot n}. \label{eq:preadd}
		\end{align}
		Now adding $(\comp_n + \varepsilon_{\delta,n})/(\estimate\eta \cdot n)$ on both sides of \eqref{eq:preadd} and choosing the estimator $\estimate\eta$ optimally in the closure of $\cG$ yields the desired result. 
	\end{proof}

	\section{Proof of Lemma \ref{lem:bernsand}}
	\begin{proof}
		\label{bernsandproof}
		Part (a) of the lemma was shown in the main body of the paper\footnote{The proof was inspired by the proof of Theorem 4 in \cite{howard2018uniform}.}. 
		Thus, we only prove part (b); we will show a slight extension, namely that for all
		$0 < u < 1$, for all $\beta > 0, u> 0$,
		\begin{equation}\label{eq:starterd}
		\sup_{\rho \leq u } \ \ \ \sup_{P: \E_P[X] = \rho, P(X \leq u)= 1}\ \ \  \E_{X \sim P} \left[e^{\beta E[X] - X - c X^2}\right] > 1 \text{\ if\ } 
		0 < c < \vartheta(u) \text{\ or \ } \beta \neq 1.
		\end{equation}
		The statement of the lemma (\ref{eq:blade}) follows as the special case for $\beta=1$, by replacing $X$ by $\eta X$ and setting $u$ to $u := \eta b < 1$.
		
		We prove this by considering the set of distributions satisfying the constraint $\E[X] = \rho$ that are supported on at most two points, 
		$${\cal P}_{\underline{x},\rho,\bar{x},u} =
		\{P: P\{\underline{x}\} + P\{\bar{x}\} = 1; \E_P[X] = \rho,
		\underline{x} \leq \bar{x} \leq u\},$$ and showing that
		\begin{equation}\label{eq:starterb}
		\sup_{\rho \leq u } \ \ \ \sup_{P \in {\cal P}_{\underline{x},\rho,\bar{x},u}
		} g_{c,\beta}(P),\ \text{\ with \ }  \  g_{c,\beta}(P) := \E_{X \sim P} \left[e^{\beta \rho - X - c X^2}\right] 
		\end{equation}
		is larger than $1$.  We first show
		that , for any $\beta \neq 1$, we can choose such a $P$ such that
		$\sup_{P \in {\cal P}_{\underline{x},\rho,\bar{x},u}
		} g_{c,\beta}(P) > 1$. To see this,  write $g_{c,\beta}(P)$ as
		$$
		p\cdot e^{-\underline{x} + \beta \rho - c \underline{x}^2 } + (1-p) e^{-\bar{x} + \beta \rho - c \bar{x}^2}
		$$
		with $\rho = \E_P[X]$. We need to  maximize this over $\rho = p \underline{x} + (1-p) \bar{x}$,
		so that in the end, we want to maximize over $0 \leq p \leq 1, \underline{u} \leq \underline{x} \leq \bar{x} \leq u$, the expression
		$$
		p\cdot e^{-\underline{x} + \beta( p \underline{x} + (1-p) \bar{x}) - c \underline{x}^2 } + (1-p) e^{-\bar{x} + \beta( p \underline{x} +(1-p) \bar{x})  - c \bar{x}^2}
		$$
		Now we write $\underline{x} = \bar{x} - a$ for some $a \geq 0$. The expression becomes
		$$
		p \cdot e^{- \beta p a +(\beta -1) \bar{x} + a  - c(\bar{x} -a)^2}
		+ (1-p) \cdot e^{- \beta pa + (\beta -1) \bar{x} - c \bar{x}^2}
		$$
		which is equal to 
		\begin{equation}\label{eq:newbasis}
		f(p,a,\bar{x}) := e^{- c \bar{x}^2 -\beta p a +(\beta -1) \bar{x}} 
		\left( p e^{a + 2 c a \bar{x} - ca^2}
		+ 1-p \right) \underset{\text{if $\beta = 1$}}{=} 
		e^{c \bar{x}^2 -pa } \left(
		p e^{a + 2 c a \bar{x} - ca^2} + 1-p \right),  \end{equation}
		where the dependency of $f$ on $c$ and $\beta$ is suppressed in the notation. At $p=1$ and $p= 0$, this simplifies to (using also $\underline{x}$ again)
		\begin{align}
		f(1,a,\bar{x}) & = e^{- c \bar{x}^2 -\beta a +(\beta -1) \bar{x}} 
		\left( e^{a + 2 c a \bar{x} - ca^2} \label{eq:fboundary1}
		\right) = e^{- c \underline{x}^2 + (\beta-1) \underline{x}}
		\underset{\text{if $\beta = 1$}}{=}  e^{- c \underline{x}^2 } \\ 
		f(0,a,\bar{x}) & = e^{- c \bar{x}^2  +(\beta -1) \bar{x}}
		\underset{\text{if $\beta = 1$}}{=}  e^{- c \bar{x}^2 }. \label{eq:fboundary0}
		\end{align}
		If $\beta < 1$, we can choose $\underline{x} = \bar{x} -a$ negative
		yet very close to $0$ making $f(1,a,\bar{x})> 1$; if $\beta > 1$, we
		can choose $\bar{x}$ positive yet very close to $0$ making
		$f(0,a,\bar{x}) > 1$. Thus, $\sup g_{c,\beta}(P)$ can be made larger
		than $1$ by $P$ satisfying the constraint if $\beta \neq 1$. This
		shows (\ref{eq:starterd}) for the case $\beta \neq 1$.  Hence, from
		now on we restrict to the case $\beta = 1$; we will further restrict
		to $\underline{x}$ and $\bar{x}$ such that
		$\underline{x} \leq 0 \leq \bar{x}$ so $\bar{x} \leq a$. We will
		determine the maximum over (\ref{eq:newbasis}) for $a \geq \bar{x}$
		and $0 \leq p \leq 1$, for each given $0 \leq \bar{x} \leq u$.  The
		partial derivatives to $p$ and $a$ are:
		\begin{align}
		\frac{\partial}{\partial p} f(p,a,\bar{x}) & = 
		e^{- c \bar{x}^2 -p a } 
		\left( \; \left( e^{a + 2 c a \bar{x} - ca^2}
		-1 \right) 
		- a \cdot
		\left( p e^{a + 2 c a \bar{x} - ca^2} + (1-p) \right) \; \right) \nonumber \\ & = 
		e^{- c \bar{x}^2 -p a } 
		\left( \;  
		e^{a + 2 c a \bar{x} - ca^2}  (1  - ap) -1 -a + a p   \right)
		\label{eq:dab} \\
		\frac{\partial}{\partial a} f(p,a,\bar{x}) & = 
		-  p \cdot
		e^{- c \bar{x}^2 -p a } 
		\left( p e^{a + 2 c a \bar{x} - ca^2}
		+ (1-p) \right) + \nonumber \\ & + e^{- c \bar{x}^2 -p a } \cdot p \cdot e^{a + 2 c a \bar{x} - ca^2} \cdot \left(1 + 2c \bar{x} - 2c a \right)
		\nonumber \\ & =  p (1-p)  \cdot e^{- c \bar{x}^2 - p a }  \cdot\left(
		-  1  + e^{a + 2 c a \bar{x} - ca^2}(1 + 2c \frac{\bar{x} - a}{1-p} ) \right). 
		\label{eq:dpb}
		\end{align}
		At $a=\bar{x}$ (\emph{i.e.} $\underline{x} = 0$), $f(p,a,\bar{x})$ simplifies to
		$$
		f(p,\bar{x},\bar{x}) = e^{- c \bar{x}^2 - p \bar{x}} \cdot (p e^{\bar{x} + c \bar{x}^2} + (1-p)) \text{\ so \ }
		f(1,\bar{x},\bar{x}) =1 
		$$
		and the partial derivative to $p$ at $(p,a,\bar{x}) = (1, \bar{x}, \bar{x})$ becomes
		\begin{equation}\label{eq:firstpar}
		e^{- c \bar{x}^2 - \bar{x} } 
		\left( (e^{\bar{x} +  c (\bar{x})^2} -1) 
		- \bar{x} e^{\bar{x} + c (\bar{x})^2} \right) = 1 - e^{- c \bar{x}^2 - \bar{x} }  - \bar{x}.
		\end{equation}
		If (\ref{eq:firstpar}) is negative, we can take $a = \bar{x}$ and $p$ slightly smaller than $1$ to get $f(p,a,\bar{x}) > 1$. This happens if and only if  
		$c$ is smaller than
		\begin{equation}\label{eq:boundy}
		\frac{- \ln (1- \bar{x}) - \bar{x}}{\bar{x}^2} = \vartheta (\bar{x}). 
		\end{equation}
		Thus, by taking $\underline{x}=0$ and $\bar{x} = a = u$, and $p$ slightly smaller than $1$ again, we get $f(p,a,\bar{x}) > 1$ if  $c < \vartheta(u)$; this shows (\ref{eq:starterd}) for the case $\beta=1$; the result is proved.
	\end{proof}

	\section{Comparison Between ``Bernstein'' Inequalities}
	\label{supp:unexpVSexp}
	\paragraph{Discussion and Proof of Our Version of Bernstein's Inequality (\ref{eq:ourbernstein}).}
	Standard versions of Bernstein's inequality (see \cite{CesaBianchiL06}, and \cite[Lemma 5.6]{erven2015fast}) can also be brought in ESI notation. 
	In particular, compared with our version they express the inequality in terms of the random variable $Y=-X$, which is then upper bounded by $b$; more importantly, they have the second moment rather than the variance on the right-hand side, resulting in a slightly worse  multiplicative factor $\kappa(2\eta b)$ instead of our $\kappa(\eta b)$; the proof is a standard one (see \cite[Lemma A.4]{CesaBianchiL06}) with trivial modifications: let $U := \eta X$ and $\bar{u} \coloneqq \eta b$. Since $\kappa(u)$ is nondecreasing in $u$ and $U \leq \bar{u}$, we have 
	$$
	\frac{e^U- U - 1}{U^2} \leq \frac{e^{\bar{u}}  - \bar{u} - 1}{\bar{u}^2},$$
	and hence 
	${e^U- U - 1} \leq \kappa(\bar{u}) U^2$. Taking expectation on both sides and using that $\ln \E[e^U] \leq \E[U] -1$, we get
	$\ln \E\left[e^{U} \right] - \E[U] \leq \kappa(\bar{u})\E[U^2]$. The result follows by exponentiating, rearranging, and using the ESI definition. 
	
	\paragraph{Comparison Between Un-expected and Empirical Bernstein Inequalities.} The proof of the following proposition demonstrates how the un-expected Bernstein inequality in Lemma \ref{lem:bernsand} together with the standard Bernstein inequality \eqref{eq:ourbernstein} imply a version of the empirical Bernstein inequality in \citep[Theorem 4]{maurer2009empirical} with slightly worse factors. However, the latter inequality cannot be used to derive our main result --- we do really require our new inequality to show Theorem~\ref{thm:main}, since we need to ``chain'' it to work with samples of length $n$ rather than $1$ in a different way. In the next proposition, we will use the following grid $\cG$ and distribution $\pi$,
	\begin{align} \label{eq:grid2} \cG \coloneqq \left\{\tfrac{1}{\nu } , \dots, \tfrac{1}{\nu^K }: K \coloneqq \Ceil{\log_\nu \left( \sqrt{\tfrac{n}{2\ln \frac{2}{\delta} }}\right) }\right\},  \ \ \text{and} \ \ \text{ $\pi=$ uniform distribution over $\cG$}. \end{align}
	for $\nu>0$. To simplify the presentation, we will use $\nu=2$ in the next proposition, albeit this may not be the optimal choice. 
	\begin{proposition}
		\label{unversusemp}
		Let $\cG$ be as in \eqref{eq:grid2} with $\rho=2$, and $Z, Z_1,\dots, Z_n$ be i.i.d random variables taking values in $[0,1]$. Then, for all $\delta \in]0,1[$, with probability at least $1-\delta$,  
		\begin{align}
		\Exp[Z] - \frac{1}{n} \sum_{i=1}^n Z_i  &\leq \left( 3\sqrt{\frac{\var_n[Z] \cdot \ln \frac{2|\cG|}{\delta}}{2n }} + \frac{11\ln \frac{2|\cG|}{\delta}}{10n} \right)\vee \frac{11 \ln \frac{2|\cG|}{\delta}}{4n} +\frac{c_{1/2}\cdot \ln \frac{2}{\delta}}{2  n}, \nonumber
		\end{align}
		where $\var_n[Z]\coloneqq  \frac{1}{n}\sum_{i=1}^n \left(Z_i - \frac{1}{n}\sum_{j=1}^n Z_j\right)^2$ is the empirical variance, $c_{1/2} \coloneqq \vartheta(1/2)/2$, and $\vartheta$ as in  Lemma \ref{lem:bernsand}. 
	\end{proposition}
	\begin{proof}
		Let $\delta \in]0,1[$. Applying Lemma \ref{lem:bernsand} to $X_i=Z_i - \Exp[Z]$, for $i\in [n]$, we get, for all $0<\eta<1/2$,
		\begin{align}
		\Exp\left[ Z \right] - Z_i \stochleq_{\eta} c_{\eta}\cdot (Z_i - \Exp[Z])^2, \quad \text{where $c_{\eta} \coloneqq \eta \cdot \vartheta(\eta )$}. \label{individual}
		\end{align}
		Applying Proposition \ref{prop:Trans}-(b) to chain \eqref{individual} for $i=1,\dots,n$, then dividing by $n$ yields
		\begin{align}
		\Exp[Z] - \frac{1}{n} \sum_{i=1}^n Z_i  &\stochleq_{n \eta}  \frac{c_{\eta} }{n}\sum_{i=1}^n  \left(Z_i - \Exp[Z]\right)^2,  \label{summed} \\
		&= c_{\eta} \cdot \var_n[Z] + c_{\eta}\cdot \left( \Exp[Z]- \frac{1}{n} \sum_{i=1}^n Z_i\right)^2, \label{eq:biasvar}
		\end{align}
		where the equality follows from the standard bias-variance decomposition. 
		Let $\cG$ and $\pi$ be as in \eqref{eq:grid2}, and let $\estimate\eta =\estimate\eta(Z_{\leq n})$ be any random estimator with support on $\cG$. By Proposition \ref{prop:randeta}, a version of \eqref{eq:biasvar} with $\eta$ is replaced by $\estimate\eta$ and $\ln (|\cG|)/(n \estimate\eta)$ added to its RHS also holds. By applying Proposition \ref{prop:drop} to this new inequality, we get, with probability at least $1-\delta$,
		\begin{align}
		\Exp[Z] - \frac{1}{n} \sum_{i=1}^n Z_i  &\leq  c_{\estimate\eta} \cdot \var_n[Z] + \frac{\ln \frac{|\cG|}{\delta}}{n \cdot \estimate\eta } + c_{\estimate\eta}\cdot \left(\Exp[Z]- \frac{1}{n} \sum_{i=1}^n Z_i\right)^2. \label{eq:withhighprob}
		\end{align}
		Now using Hoeffding's inequality \citep[Theorem 3]{maurer2009empirical}, we also have 
		\begin{align} 
		\left( \Exp[Z]- \frac{1}{n} \sum_{i=1}^n Z_i\right)^2 \leq \frac{\ln \frac{1}{\delta}}{2n}, \label{eq:hoeffding}
		\end{align}
		with probability at least $1-\delta$. Thus, by combining \eqref{eq:withhighprob} and \eqref{eq:hoeffding} via the union bound, we get that, with probability at least $1-\delta$,
		\begin{align}
		\Exp[Z] - \frac{1}{n} \sum_{i=1}^n Z_i  &\leq  \left( c_{\estimate\eta} \cdot \var_n[Z] + \frac{\ln \frac{2|\cG|}{\delta}}{n \cdot \estimate\eta } \right)+ \frac{c_{\estimate\eta} \cdot \ln \frac{2}{\delta} }{2 n}. \label{eq:tradeoff}
		\end{align}
		
		We now use the fact that for all $\eta \in ]0,1/2[$,
		\begin{align}
		c_{\eta} = \eta \cdot \vartheta(\eta) \leq \frac{\eta}{2} + \frac{11  \eta^2}{20}.  \label{eq:boundceta}
		\end{align}
		Let $\hat\eta_* \in [0,+\infty]$ be the un-constrained estimator defined by 
		\begin{align}
		\hat\eta_* \coloneqq \sqrt{\frac{2 \ln \frac{2|\cG|}{\delta}}{\var_n[Z] \cdot n}}. \nonumber
		\end{align}
		Note that by our choice of $\cG$ in \eqref{eq:grid2}, we always have $\hat\eta_* \geq \min \cG$.
		Let $\hat\eta \in ([\hat\eta_*/2, \hat\eta_*] \cap \cG) \neq \emptyset$, if $\hat\eta_* \leq 1$, and $\hat\eta =1/2$, otherwise. In the first case (\emph{i.e.} when $\hat\eta_*\leq 1$), substituting $\eta$ for $\hat\eta \in ([\hat\eta_*/2,\hat\eta_*] \cap \cG)$ in the expression between brackets in \eqref{eq:tradeoff}, and using the fact that $\hat\eta_*/2  \leq\hat\eta\leq \hat\eta_*$ and \eqref{eq:boundceta}, gives
		\begin{align}
		\label{eq:first}
		c_{\hat\eta} \cdot \var_n[Z] + \frac{\ln \frac{2 |\cG|}{\delta} }{\hat\eta \cdot n}   \leq (1+2)\sqrt{\frac{\var_n[Z] \cdot \ln \frac{2|\cG|}{\delta}}{2 n}} + \frac{11 \cdot \ln \frac{2|\cG|}{\delta}}{10 n}.
		\end{align}
		Now for the case where $\hat\eta_*\geq 1$, we substitute $\eta$ for $\hat\eta = 1/2$ in the expression between brackets in \eqref{eq:tradeoff}, and use \eqref{eq:boundceta} and the fact that $1 \leq \hat\eta_*= \sqrt{2 \ln (2|\cG|/\delta)/(\var_n[Z] \cdot n)}$, we get:
		\begin{align}
		c_{\hat\eta} \cdot \var_n[Z] +\frac{\ln \frac{2|\cG|}{\delta}}{\hat\eta \cdot n}& \leq \left(\frac{\hat\eta}{2} +\frac{11 \hat\eta^2}{20}\right)\cdot \var_n[Z] + \frac{2  \cdot \ln \frac{2|\cG|}{\delta}}{n}, \nonumber \\
		& \leq \left(\frac{\hat\eta}{2} +\frac{11  \hat\eta^2}{20}\right)\cdot \frac{2  \ln \frac{2|\cG|}{\delta}}{n} + \frac{2  \cdot \ln \frac{2|\cG|}{\delta}}{n},    \quad \text{(due to $\hat\eta_* \geq 1$)} \nonumber\\
		& =  \frac{11  \ln \frac{2|\cG|}{\delta}}{4n}, \quad \text{($\hat\eta = 1/2$)} \label{eq:last}
		\end{align}
		Combining \eqref{eq:tradeoff}, with \eqref{eq:first} and \eqref{eq:last} yields the desired results.
	\end{proof}
	\section{Additional Experiments}
	\label{supp:moreexp}
	\subsection{Informed Priors}
	\label{app:inform}
	In this section, we run the same experiments as in Section \ref{sec:experiments} of the main body, except for the following changes 
	\begin{itemize}
		\item For Maurer's bound, we use the version in our Lemma \ref{lem:fwdbwdmaurer} with informed priors.
		\item For the TS and Catoni bounds, we build a prior from the first half of the data (\emph{i.e.} we replace $P_0$ by $P(Z_{\leq m})$, where $m=n/2$) and use it to evaluate the bounds on the second half of the data. In this case, the ``posterior'' distribution is $P(Z_{> m})$, and thus the term $\KL(P_n \pipes P_0)$ is replaced by $\KL(P(Z_{>m})\pipes P(Z_{\leq m}))$.
	\end{itemize}
	Recall that $P(Z_{> m}) \equiv \cN(\hat{h}(Z_{>m}),\sigma^2 I_d)$, $P(Z_{\leq m}) \equiv \cN(\hat{h}(Z_{\leq m}),\sigma^2 I_d)$, and $P(Z_{\leq n}) \equiv \cN(\hat{h}(Z_{\leq n}),\sigma^2 I_d)$, where the variance $\sigma^2$ is learned from a geometric grid (see Section \ref{sec:experiments}); our own bound is not affected by any of these changes. The results for the synthetic and UCI datasets are reported in Figure \ref{fig:synthetic2} and Table \ref{tab:UCI2}, respectively. 
	
	\begin{figure}[ht]
		\includegraphics[clip, width=0.6\linewidth]{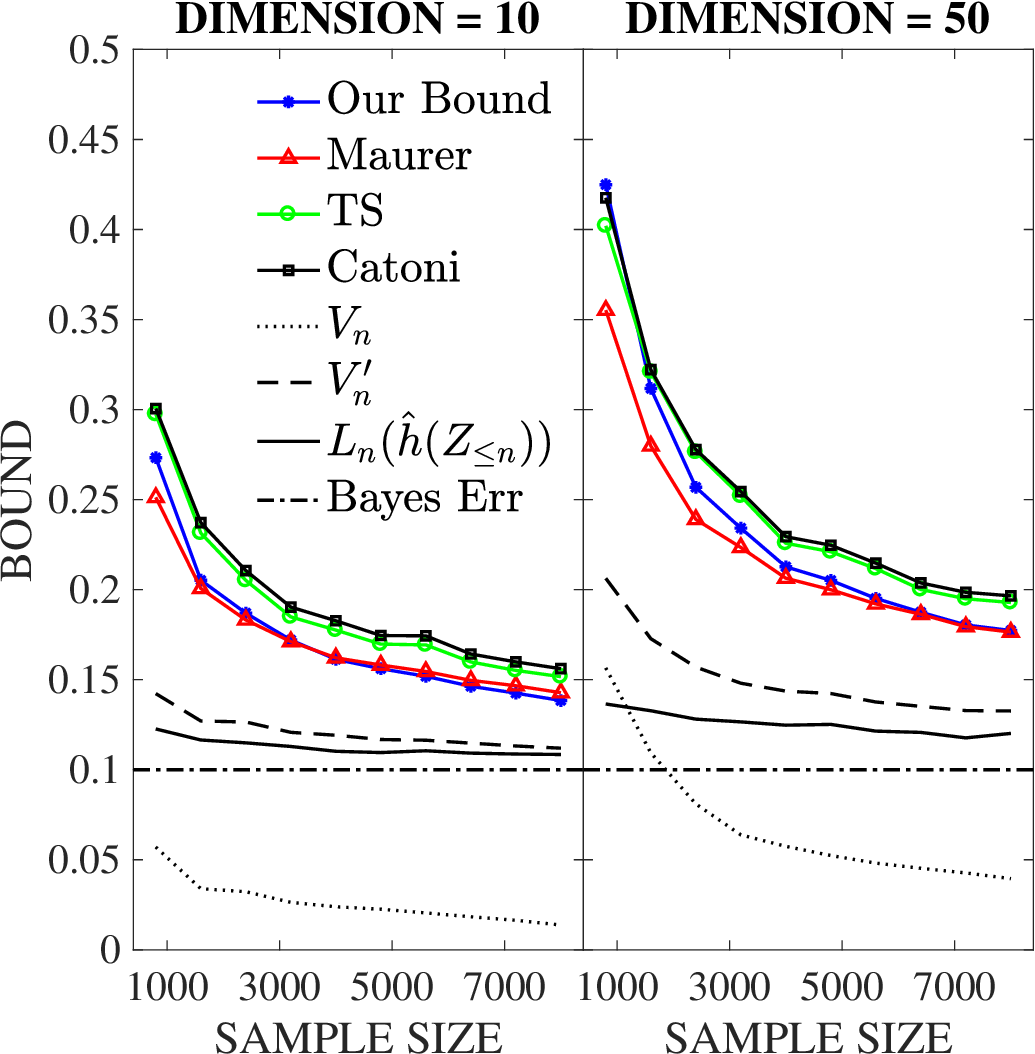} 
		\caption{Results for the synthetic data with informed priors.}
		\label{fig:synthetic2}
	\end{figure}
	
	\begin{table}[ht]
		\begin{tabular}{c c c c  c c c c}
			\hline
			Dataset & n & d   & Test error of $\hat{h}$    & Our  & Maurer  & TS & Catoni \\
			\hline
			Haberman & 244 & 3 & 0.272 & 0.52 & 0.459 & 0.501 & 0.55\\
			\hline
			Breast-C. &560 & 9 & 0.068 & 0.185 & 0.164 & 0.215 & 0.219 \\
			\hline
			Tic-Tac-Toe & 766& 27 & 0.046 & 0.19 & 0.152 & 0.202 & 0.199 \\
			\hline
			Bank-note  & 1098 & 4 & 0.058 & 0.125 & 0.117 & 0.136 & 0.143\\
			\hline
			kr-vs-kp & 2556& 73 & 0.044 & 0.107 & 0.102 & 0.123 & 0.127 \\
			\hline
			Spam-base & 3680& 57 & 0.173 & 0.293 & 0.284 & 0.317 & 0.323\\		
			\hline
			Mushroom& 6500 & 116 & 0.002 & 0.018 & 0.016 & 0.023 & 0.024\\
			\hline 
			Adult & 24130&108 & 0.168 & 0.195 & 0.198 & 0.2 & 0.203\\
			\hline
		\end{tabular}
		\caption{Results for the UCI datasets.}
		\label{tab:UCI2}
	\end{table}
	Though our bound still performs better than Catoni's and TS, Maurer's bound in Lemma \ref{lem:fwdbwdmaurer} tends to be slightly tighter than ours, especially when the sample size is small. We note, however, that the advantage of our bound has not been fully leveraged here; our bound in its full generality in Theorem \ref{thm:main} allows one to use ``online posteriors'' $(Q(Z_{>i}))$ and $(Q(Z_{< j}))$ in the $V_n$ term which converge to the one based on the full sample, \emph{i.e.} $Q(Z_{\leq n})$. We expect this to substantially improve our bound. However, we did not experiment with this due to computational reasons. 
	
	\subsection{Maurer's Bound: Informed Versus Uninformed Priors}
	In this section, we compare the performance of Maurer's bound with and without informed priors (\emph{i.e.} \eqref{eq:maurer} and \eqref{eq:maurerinformed}, respectively) on synthetic data in the same setting as Section \ref{sec:experiments}. 
	\begin{figure}[ht]
		\includegraphics[clip, width=0.6\linewidth]{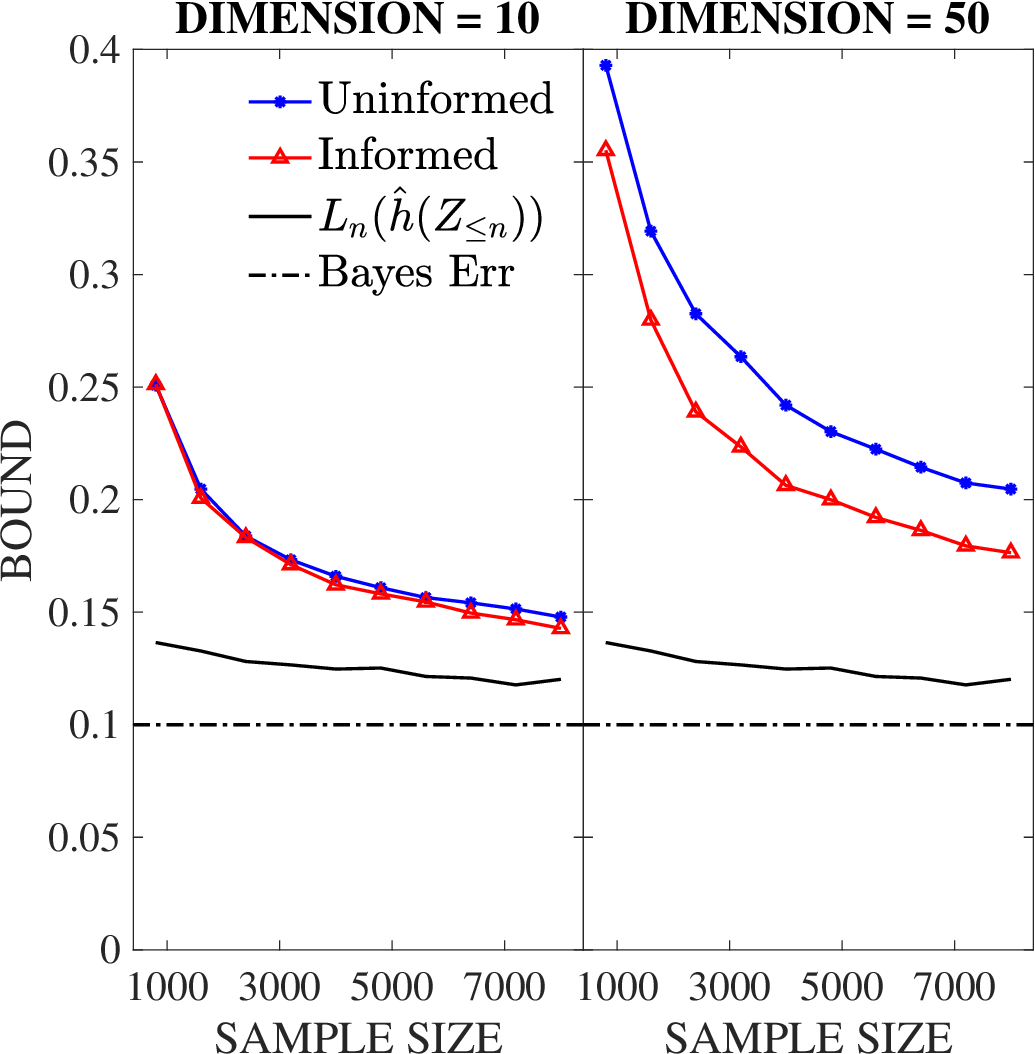} 
		\caption{Results for the synthetic data: (Blue curve) Uninformed Maurer's bound \eqref{eq:maurer}; (Red curve) Informed Maurer's bound \eqref{eq:maurerinformed}.}
		\label{fig:informedversusuninformed}
	\end{figure}
	From Figure \ref{fig:informedversusuninformed}, we see that using informed priors as in Lemma \ref{lem:fwdbwdmaurer} substantially improves Maurer's bound.  
	
	\subsection{Varying the Bayes Error and Bayes Act}
	In this subsection, we run the same synthetic experiment as in Subsection \eqref{app:inform} (\emph{i.e.} using informed priors for all bounds), except for the following changes: 
	\begin{itemize}
		\item We vary the Bayes error by varying the level of noise: we flip the labels with probability either $0.05$, $0.1$, or $0.2$ (note that in Section \ref{sec:experiments} we flipped labels with probability $0.1$).
		\item In each case, we generate the synthetic data using a randomly generated $h_*$ with coordinates uniformly sampled in the interval $[0,1]$. The reported results in Figures \ref{fig:synthetic_005_half}-\ref{fig:synthetic_010_half} are averages over 10 runs for each tested sample size.
	\end{itemize}

	\begin{figure}
		\begin{floatrow}
			\ffigbox{%
				\includegraphics[trim=0cm 0cm 0cm 0cm, clip, width=1\linewidth]{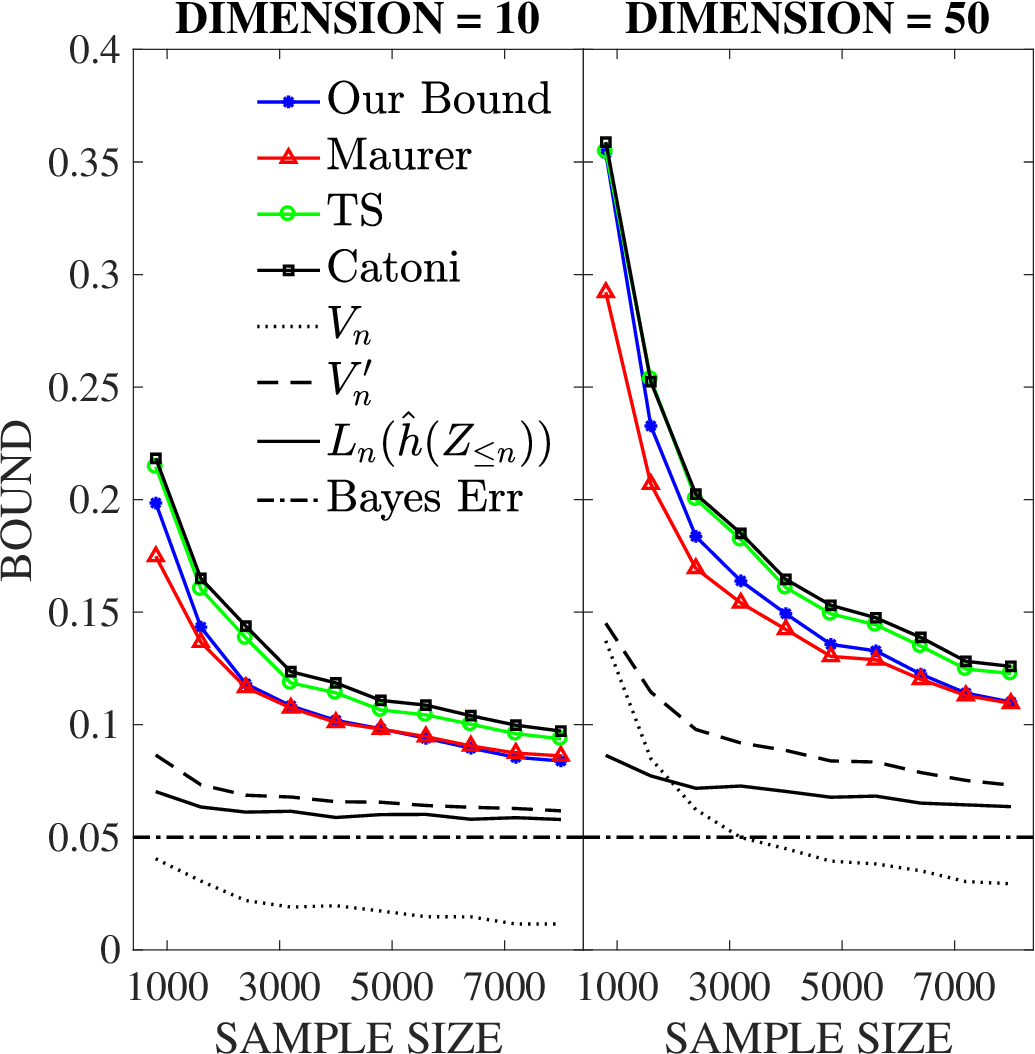} 
			}{%
				\caption{Results for the synthetic data with informed priors, randomly generated Bayes act, and Bayes error set to 0.05.}%
				\label{fig:synthetic_005_half}
			}
			\hspace{-0.cm}
			\ffigbox{%
				\includegraphics[trim=0cm 0cm 0cm 0cm, clip, width=1\linewidth]{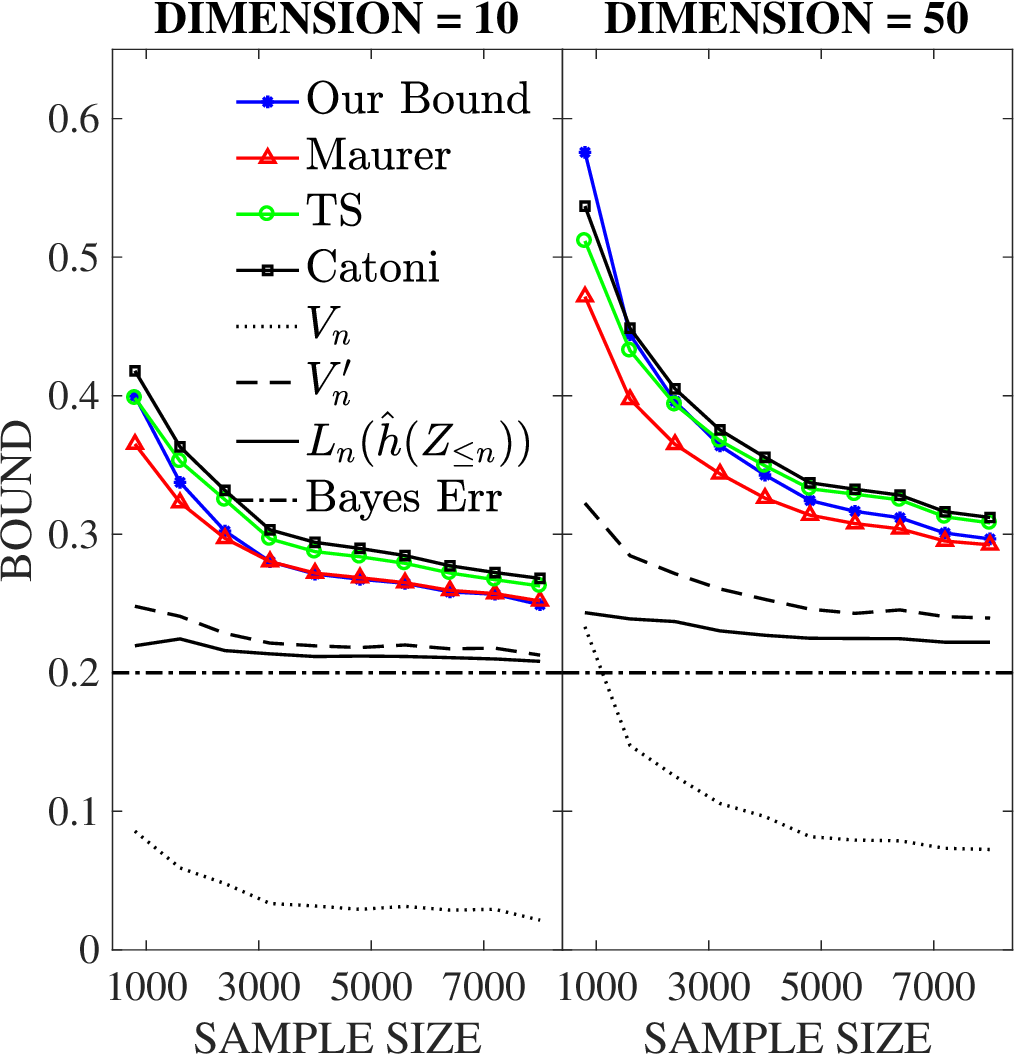} 
			}{%
				\caption{Results for the synthetic data with informed priors, randomly generated Bayes act, and Bayes error set to 0.2.}%
				\label{fig:synthetic_020_half}
			}
		\end{floatrow}
	\end{figure}
	\begin{figure}[ht]
		\includegraphics[clip, width=0.6\linewidth]{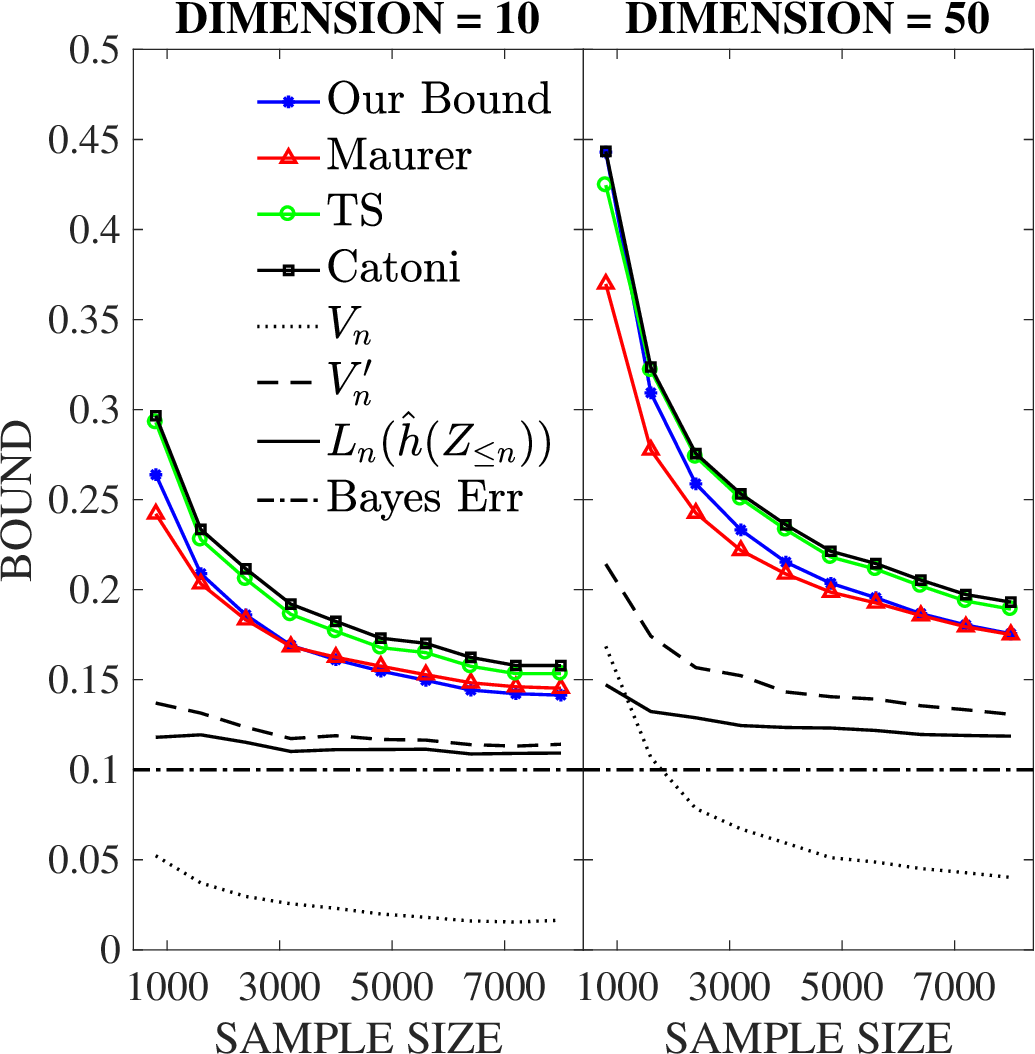} 
		\caption{Results for the synthetic data with informed priors, randomly generated Bayes act, and Bayes error set to 0.1.}
		\label{fig:synthetic_010_half}
	\end{figure}
	
\end{appendix}

\end{document}